\def\year{2019}\relax
\documentclass[letterpaper]{article} 
\usepackage{aaai19}  
\usepackage{times}  
\usepackage{helvet}  
\usepackage{courier}  
\usepackage{url}  
\usepackage{graphicx}  
\usepackage[utf8]{inputenc}
\usepackage{amssymb,amsmath,amsfonts,amsthm}
\usepackage{array,setspace,multicol,multirow}
\usepackage{graphicx}
\usepackage{enumerate,enumitem}
\usepackage{verbatim,textcomp}
\usepackage{pdfsync}
\usepackage{url}
\usepackage[dvipsnames]{xcolor}
\usepackage[symbol]{footmisc}
\usepackage[T1]{fontenc}
\DeclareSymbolFont{rsfscript}{OMS}{rsfs}{m}{n}
\DeclareSymbolFontAlphabet{\mathrsfs}{rsfscript}
\usepackage{algorithm,algorithmicx,algpseudocode}
\newtheorem{theorem}{Theorem}

\newtheorem{proposition}[theorem]{Proposition}

\theoremstyle{remark}
\newtheorem{example}{Example}
\newtheoremstyle{TheoremNum}
        {\topsep}{\topsep}              
        {\itshape}                      
        {}                              
        {\bfseries}                     
        {.}                             
        { }                             
        {\thmname{#1}\thmnote{ \bfseries #3}}
    \theoremstyle{TheoremNum}
\newtheorem{theoremnum}{Theorem}
\newtheorem{propositionnum}{Proposition}

\makeatletter
\def\namedlabel#1#2{\begingroup
   \def\@currentlabel{#2}%
   \label{#1}\endgroup
}
\makeatother

\usepackage{lipsum}
\usepackage{listings}
\lstdefinelanguage{RBG}{
	alsoletter  = {\#},
    basicstyle=\scriptsize,
    stepnumber=1,
    numbers=left,
    numbersep=8pt,
    xleftmargin=16pt,
    numberstyle=\color{gray},
    keywords={\#board,\#variables,\#pieces,\#players,\#rules,
      \#name,\#bPieces,\#diagonalMove,\#wBishopMove,\#wQueenMove,
      \#anySquare,\#turn,\#line,\#m0,\#m1,\#m2,\#m3,\#m4,\#m5,\#m6,\#m7,\#pickUpPiece,\#basicMove,\#endGame,\#checkForWin,\#fullMove,\#up8Times,\#majorPieces,\#promotePawn
    },
    keywordstyle=\bfseries\color{purple},
    keywords=[2]{
        and,or,not},
    keywordstyle=[2]\bfseries\color{black},
    morecomment=[l]{//}, 
    morecomment=[s]{/*}{*/}, 
    commentstyle=\itshape\color{cyan},
    morestring=[b]", 
    stringstyle=\color{teal},
}
\usepackage{syntax} 

\renewcommand{\O}{\mathcal{O}}

\renewcommand{\mit}[1]{\mathit{#1}}
\newcommand{\doubleto}{\to\mathrel{\mkern-13mu}\to}

\usepackage[framemethod=tikz]{mdframed}
\mdfdefinestyle{grammar}{skipabove=7pt} 

\usepackage{bibunits}

\begin{document}
%
\title{Regular Boardgames}
\author{Jakub Kowalski, Maksymilian Mika, Jakub Sutowicz, Marek Szyku{\l}a\\
Institute of Computer Science, University of Wroc{\l}aw, Wroc{\l}aw, Poland\\
jko@cs.uni.wroc.pl, mika.maksymilian@gmail.com, jakubsutowicz@gmail.com, msz@cs.uni.wroc.pl\\
}
\maketitle
\begin{abstract}
We propose a new General Game Playing (GGP) language called Regular Boardgames (RBG), which is based on the theory of regular languages. The objective of RBG is to join key properties as expressiveness, efficiency, and naturalness of the description in one GGP formalism, compensating certain drawbacks of the existing languages. This often makes RBG more suitable for various research and practical developments in GGP. While dedicated mostly for describing board games, RBG is universal for the class of all finite deterministic turn-based games with perfect information. We establish foundations of RBG, and analyze it theoretically and experimentally, focusing on the efficiency of reasoning. Regular Boardgames is the first GGP language that allows efficient encoding and playing games with complex rules and with large branching factor (e.g.\ amazons, arimaa, large chess variants, go, international checkers, paper soccer).
\end{abstract}



\section{Introduction}

Since 1959, when the famous General Problem Solver was proposed \cite{Newell1959Report}, a new challenge for the Artificial Intelligence has been established.
As an alternative for the research trying to solve particular human-created games, like chess \cite{Campbell2002Deep}, checkers \cite{Schaeffer2007Checkers}, or go \cite{Silver2016Mastering}, this new challenge aims for creating less specialized algorithms able to operate on a large domain of problems.

In this trend, the General Game Playing (GGP) domain was introduced in \cite{Pitrat1968Realization}. Using games as a testbed, GGP tries to construct universal algorithms that perform well in various situations and environments. Pell's METAGAMER program from 1992 was able to play and also to generate a variety of simplistic chess-like games \cite{Pell1992METAGAME}. 
The modern era of GGP started in 2005, with the annual International General Game Playing Competition (IGGPC) announced by Stanford's Logic Group \cite{Genesereth2005General}. Since that time, GGP became a well-established domain of AI research consisting of multiple domains and competitions \cite{Swiechowski2015Recent}.

The key of every GGP approach is a well-defined domain of games (problems) that programs are trying to play (solve). 
Such a class of games should be properly formalized, so the language used to encode game rules will be sound. 
For practical reasons, the language should be both easy to process by a computer and human-readable, so conciseness and simplicity are in demand.
Finally, the defined domain has to be broad enough to provide an appropriate level of challenge.

\noindent\textbf{Existing languages.}
METAGAME is based on Dawson's Theory of Movements \cite{Dickins1971AGuide} that formalizes types of piece behavior in chess-like games.
It contains a large number of predefined keywords describing allowed piece rules, including exceptional ones like promotions or possessing captured pieces.
Despite that, METAGAME's expressiveness remains very limited and it cannot encode some core mechanics of chess or shogi (e.g.\ castling, en passant).

Ludi system was designed solely for the sake of procedural generation of games from a restricted domain \cite{Browne2010Evolutionary}.
The underlying GGP language is based on a broad set of predefined concepts, which makes designing a potential player a laborious task.
It describes in a high-level manner a rich family of combinatorial games (which, however, does not include e.g.\ the complete rules of chess).

Simplified Boardgames \cite{Bjornsson2012Learning} describes chess-like games using regular expressions to encode movement of pieces. It overcomes some of the METAGAME's limitations and does not require extended vocabulary. However, as allowed expressions are simplistic and applied only to one piece at once, it cannot express any non-standard behavior (e.g.\ promotions or piece addition). Actually, the rules of almost all popular board games (except breakthrough) cannot be fully expressed in this language.

GDL~\cite{Love2008General}, used in IGGPC, can describe any turn-based, finite, and deterministic $n$-player game with perfect information. It is a high-level, strictly declarative language, based on Datalog \cite{Abiteboul1995Foundations}.
GDL does not provide any predefined functions.
Every predicate encoding the game structure like a board or a card deck, or even arithmetic operators, must be defined explicitly from scratch.
This makes game descriptions long and hard to understand, and their processing is computationally very expensive, as it requires logic resolution.
In fact, many games expressible in GDL could not be played by any program at a decent level or would be unplayable at all due to computational cost.
For instance, features like longest ride in checkers or capturing in go are difficult and inefficient to implement.
Thus, in such cases, only simplified rules are encoded.

The generality of GDL provides a very high level of challenge and led to many important contributions \cite{Genesereth2014General}, especially in Monte Carlo Tree Search enhancements \cite{Finnsson2008Simulation,Finnsson2010Learning}. However, the downside of domain difficulty is a fairly low number of existing GGP players and competition entries.
GDL extensions, e.g.\ GDL-II \cite{Schiffel2014Representing}, which removes some of the language restrictions, are even more difficult to handle.

TOSS \cite{Kaiser2011FirstOrder}, proposed as a GDL alternative (it is possible to translate GDL into TOSS), is based on first-order logic with counting. The language structure allows easier analysis of games, as it is possible to generate heuristics from existential goal formulas.

On the other hand, VGDL \cite{Schaul2013AVideo} used in currently very popular General Video Game AI Competition \cite{Perez2016General}, is strictly focused on representing real-time Atari-like video games (similarly as the Arcade Learning Environment \cite{Bellemare2013TheArcade}).
Instead of game descriptions, the competition framework provides a forward model that allows, and simultaneously forces, simulation-based approaches to play the game.
By limiting knowledge-based approaches, which can learn how to play the game by analyzing its rules, the competition in some sense contradicts the idea of general game playing as stated by Stanford's GGP.

\noindent\textbf{Our contribution.}
The existing languages are designed for different purposes and come with different shortcomings concerning key issues as expressiveness, efficiency, and structural description.
We introduce \emph{Regular Boardgames} (\emph{RBG}), a new GGP language, which presents an original view on formalizing game rules, employing both new ideas and the best ones from the other languages.
Its main goal is to allow effective computation of complex games, while at the same time being universal and allowing concise and easy to process game descriptions that intuitively correspond to the game structure.
The base concept is a use of regular languages to describe legal sequences of actions that players can apply to the game state.
While descriptions are intuitive, RBG requires non-trivial formal basis to be well defined.

In this work, we formally introduce the language and provide its theoretical foundations.
Our experiments concern the efficiency of reasoning.
The basic version presented here describes perfect information deterministic games, but it can be extended by randomness and imperfect information, and is a basis for developing even more efficient languages.
The full version of this paper is available at~\cite{Kowalski2018RBGextended}.

\section{Regular Boardgames Language}\label{sec:definition}


An \emph{abstract RBG description} is a 7-tuple $\mathcal{G}=(\mit{Players}, \mit{Board}, \mit{Pieces},\mit{Variables}, \mit{Bounds}, \mit{InitialState},\\\mit{Rules})$.
It is a complete formalization of a board game.

\noindent\textbf{Players}.
$\mathit{Players}$ is a finite non-empty set of \emph{players}.
For instance, for chess we would have $\mathit{Players}=\{\mit{white},\mit{black}\}$.

\noindent\textbf{Board}.
$\mathit{Board}$ is a representation of the board without pieces, i.e., the squares together with their neighborhood.
It is a static environment, which does not change during a play.
Formally, $\mathit{Board}$ is a 3-tuple $\left(\mathit{Vertices},\mathit{Dirs},\delta\right)$, which describes a finite directed multigraph with labeled edges.
$\mathit{Vertices}$ is its set of vertices, which represent the usual squares or other elementary parts of the board where pieces can be put.
The edges of the graph have assigned labels from a finite set $\mathit{Dirs}$, whose elements are called \emph{directions}.
For every $v \in \mathit{Vertices}$, we have at most one outgoing edge from $v$ with the same label.
Hence, the edges are defined by a function $\delta\colon \mathit{Vertices} \times \mathit{Dirs} \to \mathit{Vertices} \cup \{\bot\}$, where $\delta(v,d)$ is the ending vertex of the unique outgoing edge from $v$ labeled by $d$, or $\delta(v,d)=\bot$ when such an edge does not exist.
The uniqueness of the outgoing edges with the same label will allow walking on the graph $\mathit{Board}$ by following some directions.

The usual $8\times 8$ chessboard can be defined as follows:
The set $\mathit{Vertices} = \{(i,j)\mid 1 \le i,j \le 8\}$ represents the 64 squares, $\mathit{Dirs} = \{\mit{left},\mit{right},\mit{up},\mit{down}\}$ is the set of the four directions, and the edges are given by $\delta((i,j),\mit{left}) = (i-1,j)$ for $i \ge 2$, $\delta((1,j),\mit{left})=\bot$, and analogously for the other three  directions.

\noindent\textbf{Pieces}. 
$\mathit{Pieces}$ is a finite non-empty set whose elements are called \emph{pieces}; they represent the usual concept of the elements that may be placed on the board.
We assume that $\mathit{Pieces} \cap \mathit{Players} = \emptyset$.
During a play, there is a single piece assigned to every vertex.
Hence, in most game descriptions, there is a special piece in $\mathit{Pieces}$ denoting an empty square.
For instance, for chess we can have $\mathit{Pieces}=\{\mit{empty},\mit{wPawn},\mit{bPawn},\mit{wKnight},\mit{bKnight},\ldots\}$.

\noindent\textbf{Variables}.
There are many board games where game states are not restricted to the board configuration.
Counters, storages, or flags are commonly found; even chess requires abstract flags for allowing castling or en passant.
To express such features naturally and without loss of efficiency, we add the general concept of variables.
Each variable has a bounded domain and stores an integer from $0$ to its maximum.
The second purpose of having variables is to provide the outcome of a play in a universal way, by allowing to assign an arbitrary score for every player.
Hence, for every player, we have a variable whose value stores the \emph{score} of this player.
Formally, $\mathit{Variables}$ is a finite set of \emph{variables}.
We assume $\mathit{Variables} \cap \mathit{Pieces} = \emptyset$ and, since each player has a variable storing his score, $\mathit{Players} \subseteq \mathit{Variables}$.

\noindent\textbf{Bounds}.
$\mathit{Bounds}\colon \mathit{Variables} \to \mathbb{N}$ is a function specifying the maximum values separately for every variable.

\noindent\textbf{Game state}.
A \emph{semi-state} is a 4-tuple $S=(\mathit{player},P,V,s)$, where
$\mathit{player} \in \mathit{Players}$ is the \emph{current player}, $P\colon \mathit{Vertices} \to \mathit{Pieces}$ is a complete assignment specifying the pieces that are currently on the board's vertices, $V\colon \mathit{Variables} \to \mathbb{N}$ is a complete assignment specifying the current values of the variables, and $s \in \mathit{Vertices}$ is the \emph{current position}.
In every semi-state, we will have $V(v) \in \{0,\ldots,\mathit{Bounds}(v)\}$ for all $v \in \mathit{Variables}$.

A \emph{game state} is a semi-state $S$ with additionally the \emph{rules index} $r \in \mathbb{N}$. It indicates which part of the rules applies currently (explained later).
We denote the game state with semi-state $S$ and rules index $r$ by $S_r$; it contains all information that may change during a play.

\noindent\textbf{Initial state}. $\mit{InitialState}$ is an arbitrary semi-state. A play begins from the game state $\mit{InitialState}_0$.

\noindent\textbf{Actions}.
An \emph{action} is an elementary operation that can be applied to a semi-state $S$. It modifies $S$ and/or verifies some condition.
Depending on $\mathcal{S}$, an action may be \emph{valid} or not.
The resulted semi-state after applying an action $a$ is denoted by $S\cdot a$, which is a copy of $S$ additionally modified as defined for the particular action $a$.
The actions are:
\begin{enumerate}[wide]
\item[1.] \textbf{Shift}: denoted by $\mit{dir} \in \mathit{Dirs}$.
This action changes the position $s$ to $\delta(s,\mathit{dir})$.
It is valid only if $\delta(s,\mit{dir}) \neq \bot$.

\item[2.] \textbf{On}: denoted by a subset $X \subseteq \mathit{Pieces}$.
This action checks if $P(s)\in X$, i.e.,\ if the piece on the current position is from $X$.
It does not modify the semi-state and is valid only if this condition holds.
Note that the empty set $\emptyset$ is never valid and $\mathit{Pieces}$ is always valid.

\item[3.] \textbf{Off}: denoted by $[x]$ for $x \in \mathit{Pieces}$.
This action sets $P(s)=x$, i.e.,\ the next semi-state $S\cdot a$ contains piece $x$ on the current square (the previous piece is replaced).

\item[4.] \textbf{Assignment}: denoted by $[\$\,v=e]$ for $v \in \mathit{Variables}$ and $e$ being an arithmetic expression.
An \emph{arithmetic expression} is either a single value $r \in \mathbb{Z} \cup \mathit{Variables} \cup \mathit{Pieces}$, or (recursively) $e_1 \oplus e_2$, where $e_1$, $e_2$ are arithmetic expressions and $\oplus \in \{+,-,\cdot,/\}$ is a binary arithmetic operation (addition, subtraction, multiplication, and integer division).
The expression $e$ is evaluated as follows.
An $r \in \mathit{Variables}$ takes its value $V(r)$ in the current semi-state, and $r \in \mathit{Pieces}$ is evaluated to the number of pieces $r$ currently on the board, i.e.\ the cardinality $|\{s \in \mathit{Squares} \mid P(s)=r\}|$. Operators $+,-,\cdot,/\}$ are evaluated naturally.
If the value of $e$ is in $\{0,\ldots,\mit{Bounds}(v)\}$ then it is assigned to $v$ and the action is valid; otherwise the action is not valid.

\item[5.] \textbf{Comparison}: denoted by $\{\$\,e_1 \otimes e_2\}$, where $e_1,e_2$ are arithmetic expressions defined as above, and $\otimes \in \{<,\leq,=,\neq,>,\geq\}$ is a relational operator on integers.
This action is valid only if after evaluating the expressions, the relation is true.
It does not modify the semi-state.

\item[6.] \textbf{Switch}: ${\to}{p}$, where $p \in \mathit{Players}$.
This action changes the current player $\mathit{player}$ to $p$.
It is always valid.
\end{enumerate}
Offs, assignments, and switches are called \emph{modifiers}.

Given a semi-state $S$, one can perform a sequence of actions $a_1 \dots a_k$, which are applied successively to $S$, and results in $S \cdot a_1 \cdot \ldots \cdot a_k$.
The sequence is \emph{valid} for $S$ if all the actions are valid when they are applied.

\begin{example} One of the moves of a white knight in chess can be realized by the following sequence:
$\{\mit{wKnight}\}\,[\mit{empty}]\,\mit{left}\,\mit{up}\,\mit{up}\,\{\mit{empty}\}\,[\mit{wKnight}]\,{\to}{\mit{black}}$.
The first action (on) checks if there is a white knight on the current square.
Then this square becomes empty (off), and we change the current square (shift) three times to the final destination.
The destination square is checked for emptiness (on), and a white knight is placed on it (off).
Finally, the black player takes the control (switch).
\end{example}

\noindent\textbf{Rules}.
If action is valid, it only means that it is applicable.
An action is \emph{legal} for a game state $S_r$ if it is a valid action for $S$ that also conform the rules, which depends on $r$.

To give a formal definition of legal actions, we will use a bit of the theory of regular languages.
We deal with regular expressions whose input symbols are actions and with the three usual operators: concatenation $\cdot$ (the dot is omitted usually), sum $+$, and the Kleene star ${}^*$.
We do not need the empty word nor the empty language symbols.
Given a regular expression $E$, by $\mathcal{L}(E)$ we denote the regular language defined by $E$.
This language specifies allowed sequences of actions that can be applied to a game state.

\begin{example}
The regular expression specifying all white knight's moves in chess can look as follows:\\
${\;}(\mit{left}^*\,{+}\,\mit{right}^*)\,(\mit{up}^*\,{+}\,\mit{down}^*)\,\{\mit{wKnight}\}\,[\mit{empty}] \\
{\;}(\mit{left}\,\mit{left}\,\mit{down}\,{+}\,\mit{left}\,\mit{left}\,\mit{up}\,{+}\,\mit{left}\, \mit{down}\,\mit{down}\,{+}\,\hdots) \\
{\;}\{\mit{empty},\mit{bPawn},\mit{bKnight},\ldots\}\,[\mit{wKnight}]\,{\to}\mit{black}$. \\
The words from the language begin with an arbitrary number of either $\mit{left}$ or $\mit{right}$ shift actions, which are followed by an arbitrary number of $\mit{up}$ or $\mit{down}$ actions.
By choosing a suitable number of repetitions, we can obtain the semi-state with an arbitrarily selected current square.
The next two actions check if the current square contains a white knight and make it empty.
Then, we can select any of the eight knight's direction patterns.
After that, the current square is checked for if it is empty or contains a black piece to capture.
Finally, a white knight is placed and the black player takes the control.
\end{example}

Regular expressions define sets of allowed sequences of actions applied from the beginning of the game, but we want to define them also when we have already applied a part of such a sequence.
For this, we have the rules position in a game state, which indicates the current position in a regular expression.
Let $\mathrm{pref}(L)$ be the language of all prefixes of the words from a language $L$.
When the rules index is $0$ -- the beginning of the expression --  we will have the property that every allowed sequence is from $\mathrm{pref}(\mathcal{L}(E))$.

However, after applying some actions, our rules index will be changed so that the next available actions will form a continuation: concatenated to the previously applied actions, will be a prefix of a word from the language.

By $\hat{E}$ we denote the regular expression where all actions are subsequently numbered starting from $1$; thus they become pairs $a_i$ of an action $a$ and an index $i$.
For example, if $E=\mit{up}\,\mit{left} + [x]\,(\mit{up}\,[y])^*$ then $\hat{E}=\mit{up}_1\,\mit{left}_2 + [x]_3\,(\mit{up}_4\,[y]_5)^*$.
Indexing is used to distinguish the action occurrences in a regular expression when the same action appears multiple times because our rules index must point a unique position.
Hence we will be applying indexed actions to game states, similarly as non-indexed actions are applied to semi-states.
Suppose that we have already applied a word $u$; then we can apply any word $v$ such that $uv \in \mathrm{pref}(\mathcal{L}(\hat{E}))$.
The set of words $w$ such that $uw \in \mathcal{L}(\hat{E})$ is commonly known as a \emph{derivative} of $\mathcal{L}(\hat{E})$ by $u$ \cite{Brz64}; it is also a regular language.
Because the indexed actions in $\hat{E}$ occur uniquely, this set is completely determined by the last indexed action $a_i$ in $u=u'a_i$:
\begin{proposition}\label{pro:unique-derivative}
For every indexed action $a_i$, the non-empty derivatives of $\mathcal{L}(\hat{E})$ by $u' a_i$ are the same for every word $u'$.
\end{proposition}
\noindent We denote this derivative by $\mathcal{L}(\hat{E})_i$, and when $u$ is empty, we define $\mathcal{L}(\hat{E})_0 = \mathcal{L}(\hat{E})$.
Following our previous example, if $\hat{E}=\mit{up}_1\,\mit{left}_2 + [x]_3\,(\mit{up}_4\,[y]_5)^*$, then $\mathcal{L}(\hat{E})_4$ is the language defined by $[y]_5\,(\mit{up}_4\,[y]_5)^*$.

Finally, we define: for a game state $S_r$ under a regular expression $E$, an indexed action $a_i$ is \emph{legal} if $a$ is valid for semi-state $S$ and $a_i \in \mathrm{pref}(\mathcal{L}(\hat{E})_r)$.
The resulted game state is $(S\cdot a)_i$, i.e.,\ consists of the resulted semi-state and the index $i$ of the last applied action.
This definition is extended naturally to sequences of indexed actions.

$\mit{Rules}$ is a regular expression as above.
A \emph{move sequence} is an action sequence with exactly one switch, which appears at the end.
A play of the game starts from $\mit{InitialState}_0$.
The current player applies to the current game state a legal move sequence under $\mit{Rules}$, which defines his \emph{move}.
The play ends when there is no legal move sequence.


We finally add two more advanced elements.

\noindent\textbf{Patterns}.
There is another kind of action that can check more sophisticated conditions.
\begin{enumerate}[wide]
\item[7.] \textbf{Pattern}: denoted by either $\{? M\}$ or $\{! M\}$, where $M$ is a regular expression without switches.
$\{? M\}$ is valid for a semi-state $S$ if and only if there exists a legal sequence of actions under $M$ for $S_0$ (equivalently if there is a valid sequence of actions for $S$ from $\mathcal{L}(M)$).
$\{! M\}$ is the negated version.
These actions do not modify the semi-state.
\end{enumerate}
Patterns can be nested; thus $M$ can contain their own patterns and so on.
\begin{example}
Using a pattern, we can easily implement the special chess rule that every legal move has to leave the king not checked, by ending the white's sequences with:\\
$\big(!\,(\mathrm{standard\ black\ actions})\,\{\$\,\mit{wKing}=0\}\big)\ {\to}{\mit{black}}$.\\
Suppose that ``$(\mathrm{standard\ black\ moves})$'' stands for all possible black's action sequences (move a pawn, move a knight, etc.) respectively.
Then $(!\ \hdots)$ checks if the black player can capture the white king. If so, the pattern is not valid, hence a sequence of the white player containing it is illegal.
\end{example}

\noindent\textbf{Keeper}.
There is a special player called \emph{keeper}, who performs actions that do not belong to regular players but to the game manager, e.g.\ maintaining procedures, scores assignment, ending the play.
The keeper can have many legal sequences of actions, but we will admit only the \emph{proper} game descriptions, where his choice does not matter: the resulted game state must be the same regardless of the choice.
For example, when the keeper removes a group of pieces, he may do this in any order as long as the final game state does not depend on it.
Hence, the game manager can use an arbitrary strategy for the keeper, e.g.\ apply the first found sequence.

The keeper is an important part of the language for efficiency reasons since he can be used to split players' action into selection and application parts.
Hence, the keeper can significantly improve the efficiency of computing all legal moves (e.g.\ in MCTS), because player sequences can be much shorter and we do not have to compute all their applications when we do not need to know all next game states.
For example, in reversi, the player during his turn just puts a pawn on the board, which completely determines his move, and after that, the keeper swaps the color of the opponent's pawns; then the next player takes control.

Formally, we assume that there is the unique object $\mathfrak{Keeper} \in \mit{Players}$ representing the keeper, and a double arrow ${\doubleto}$ denotes a switch to the keeper.
\begin{example}
A typical keeper usage is to check winning condition and end the play.\\
${\;}{\doubleto}\,\big(\{?\,\mathrm{white\ wins}\}\,[\$\,\mit{white}{=}100]\,[\$\,\mit{black}{=}0]\,{\doubleto}\,\emptyset\\
{\;}+ \{!\,\mathrm{white\ wins}\}\,{\to}black\big)$.\\
This fragment is to be executed right after a white's move.
The first option checks if white has already won the game (subexpression ``white wins''), sets the scores, and continues with the keeper that has the empty \emph{on} action $\emptyset$, which is always illegal thus ends the play.
Note that the fragment is deterministic, i.e., the keeper has exactly one legal sequence.
\end{example}

\begin{figure}
\lstset{numbers=left,xleftmargin=16pt}
\begin{lstlisting}
#players = white(100), black(100) // 0-100 scores
#pieces = e, w, b
#variables = // no additional variables
#board = rectangle(up,down,left,right,
         [b, b, b, b, b, b, b, b]
         [b, b, b, b, b, b, b, b]
         [e, e, e, e, e, e, e, e]
         [e, e, e, e, e, e, e, e]
         [e, e, e, e, e, e, e, e]
         [e, e, e, e, e, e, e, e]
         [w, w, w, w, w, w, w, w]
         [w, w, w, w, w, w, w, w])
#anySquare = ((up* + down*)(left* + right*))
#turn(me; myPawn; opp; oppPawn; forward) =
  anySquare {myPawn}     // select any own pawn
  [e] forward ({e} + (left+right) {e,oppPawn})
  ->> [myPawn]           // keeper continues
  [$ me=100] [$ opp=0]   // win if the play ends
  (   {! forward} ->> {} // if the last line then end
    + {? forward}->opp) // otherwise continue
#rules = ->white (
    turn(white; w; black; b; up)
    turn(black; b; white; w; down)
  )* // repeat moves alternatingly
\end{lstlisting}
\caption{The complete RBG description of breakthrough.}\label{fig:breakthrough}
\end{figure}

\subsection{RBG Language}

The simplest version of the RBG language is \emph{low-level RBG} (\emph{LL-RBG}), which directly represents an abstract RBG description in the text. 
It is to be given as an input for programs (agents, game manager), thus it is simple and easy to process.
An extension of LL-RBG is the \emph{high-level RBG} (HL-RBG), which allows more concise and human-readable descriptions. HL-RBG can be separately converted to LL-RBG.
This split joins human readability with machine processability and allows to further develop more extensions in HL-RBG without the need to modify implementations.
The technical syntax specification is given in~\cite{Kowalski2018RBGextended}, and here we give an overall view.

\noindent\textbf{LL-RBG.}
In LL-RBG there are a few definitions of the form \lstinline{#name = definition}.
We have \lstinline{#board} specifying $\mit{Board}$ together with the initial pieces assignment, \lstinline{#players} and \lstinline{#variables} specifying the sets $\mit{Players}$ and $\mit{Variables}$ together with $\mit{Bounds}$, \lstinline{#pieces} specifying $\mit{Pieces}$, and \lstinline{#rules} defining the regular expression $\mit{Rules}$.
$\mit{InitialState}$ is the semi state where all variables are set to $0$, the current player is the keeper, and the current position and the pieces assignment are defined by \lstinline{#board}.
The simplification of defining $\mit{InitialState}$ is not a restriction since we can set any state at the beginning of $\mit{Rules}$.

\noindent\textbf{HL-RBG.}
In high-level RBG we add a simple substitution C-like macro system.
A macro can have a fixed number of parameters and is defined by \lstinline{#name(p1;...;pk) = definition} (with $k$ parameters) or \lstinline{#name = definition} (without parameters).
After the definition, every further occurrence of \textit{name} is replaced with \textit{definition}, where additionally every parameter occurrence in \textit{definition} is replaced with the values provided.
There are a few other HL-RBG extensions over LL-RBG, such as predefined functions to generate typical boards (e.g.\ rectangular, hexagonal, cubical, etc.).

A complete example of game \emph{breakthrough} in HL-RBG is given in Fig.~\ref{fig:breakthrough}.
For instance, the corresponding expression in LL-RBG obtained by unrolling $\mit{turn}$ macro in line~22 is:
\lstset{numbers=none,xleftmargin=0pt}
\begin{lstlisting}
((up*+down*)(left*+right*)) {w} [e] up
({e} + (left+right) {e,b}) ->>
[w][$white=100][$black=0] ({!up} ->> {} + {?up} ->opp)}
\end{lstlisting}
Note that the placement of the moved pawn is postponed to line~18.
The keeper performs this action instead of the player since the move is already defined in line~17.


\subsection{Proper RBG Description and Transition Model}

We state two conditions that a \emph{proper} RBG description must satisfy.
They will ensure that the game is finite, well defined, and also allow reasoning algorithms to be more efficient.

In the first condition, we bound the number of modifiers that can be applied during a play, including entering patterns.
This implies that we cannot reach the same game state after applying a modifier in the meantime and that every play will eventually finish.
For example, the simplest forbidden construction is $[x]^*$; however, $(\mit{left}\,[x])^*$ is allowed if the number of valid repetitions of $\mit{left}$ shift action is bounded.

Let \emph{straightness} of a word $w$ be the number of modifier occurrences in $w$.
The \emph{straightness} of a language $L$ is the maximum of the straightnesses of all $w \in L$; if the maximum does not exist, the straightness of $L$ is infinite.
Note that the valid sequences for $\mit{InitialState}$ from $\mathcal{L}(\mit{Rules})$ describe all possible plays.
However, to take into account patterns, we need to introduce one more definition.
For a semi-state $S$ and a language $L$ of non-indexed action sequences, we define the \emph{application language}, which consists of all valid sequences that we could apply when starting from $S$.
This includes all valid sequences from $\mathrm{pref}(L)$ and also the valid sequences when we are allowed to ``go inside'' a pattern.
Formally, $\mathrm{app}(S,L)$ is defined recursively by:
\begin{gather*}
\mathrm{app}(S,L) = \{u\text{ is a valid word for $S$ from }\mathrm{pref}(L)\}\\ \cup\ \{uv \mid u \text{ is a valid word for $S$ from }\mathrm{pref}(L),\\
u\{? M\}\text{ or }u\{! M\} \in \mathrm{pref}(L),
v \in \mathrm{app}(S\cdot u,\mathcal{L}(M))\}.
\end{gather*}
Therefore, in $\mathrm{app}(S,L)$ there are valid sequences of the form $u_0 u_1 \dots u_h$, where a $u_i$ is a valid prefix of a word from a pattern language nested at depth $i$ (and from $L$ for $i=0$).
We require that for the initial state and the rules, the straightness of the application language is finite.

The second condition states that the keeper strategy does not matter as his actions always eventually yield the same game state when another player takes control.
Formally, a game state $S_r$ is \emph{reachable} if there exists a legal sequence for $\mit{InitialState}_0$ under $\mit{Rules}$ that yields $S_r$.
For a game state, by its \emph{keeper completion} we mean a game state after applying any legal move sequence as long as the current player is the keeper and there exists such a sequence.
Applying a move sequence can be repeated several times, and if the current player is not the keeper, the keeper completion is the very same game state. 
In Fig.~\ref{fig:breakthrough}, the keeper in lines~19--20 always has exactly one choice, depending on whether we can perform $\mit{forward}$ shift.
However, the construction ${\doubleto}(\mit{left}+\mit{right})\,{\to}p$ could possibly yield two keeper completions differing by the position, and then it is incorrect.

Finally, an RBG description is \emph{proper} if:
\begin{enumerate}
\item The straightness of $\mathrm{app}(\mit{InitialState},\mathcal{L}(\mit{Rules}))$ is finite.
\item For every reachable game state, there is exactly one keeper completion.
\end{enumerate}

\noindent Now we define precisely the game tree represented by an RBG description, which is important, e.g.\ for drawing moves at random during a Monte-Carlo search.

A \emph{move} is a sequence of pairs $(i,v)$, where $i$ is the index of a modifier in $\hat{\mit{Rules}}$ and $s \in \mathit{Vertices}$ is the position where the modifier is applied; the last indexed modifier must be a switch, and there cannot be other switches.
Every legal move sequence defines a legal move in a natural way.
The number of legal move sequences can be infinite (e.g.\ by $(\mit{up}+\mit{down})^*$), but due to condition~(1), there is a finite number of moves.
For example, in Fig.~\ref{fig:breakthrough}, for the keeper completion of the initial state, the white player has exactly $22$ moves (containing the indices of $[e]$ and $\doubleto$).

The game tree in RBG is constructed as follows.
The root is the keeper completion of $\mit{InitialState}$.
For every legal move of a node (game state), we have an edge to the keeper completion of the game state obtained by applying that move.
The leaves are the nodes without a legal move, where the outcome is the player scores stored in their variables.
Note that in this way we do not count keeper game states (unless they are leaves), which are only auxiliary.

\section{Expressiveness and Complexity}

\noindent\textbf{Universality}.
RBG can describe every finite deterministic game with perfect information (without simultaneous moves, which are a form of imperfect information).
To show this formally, we can follow the definition of extensive-form games \cite{Rasmusen1994Games}, which has been used to show that GDL and GDL-II are universal for their classes \cite{Thielscher2011TheGeneral}, and prove that in RBG we can define an arbitrary finite game tree.
It is enough to encode a game tree in $\mit{Rules}$, where for every tree node we create a switch.
\begin{theorem}\label{thm:rbg-is-universal}
RBG is universal for the class of finite deterministic games with full information.
\end{theorem}

\begin{table*}[htb]\centering\small\renewcommand{\arraystretch}{1.2}
\newcommand{\rowt}[1]{\multirow{2}{*}{#1}}
\newcommand{\colt}[1]{\multicolumn{2}{c|}{#1}}
\caption{Complexity of basic decision problems.}\label{tab:complexity}
\begin{tabular}{|l|c|c|c|c|}\hline
\bf Subclass                & {\bf Legal move?} (Problem~\ref{pbm:haslegal}) & {\bf Winning strategy?} (Problem~\ref{pbm:winning}) & {\bf Proper description?} (Problem~\ref{pbm:proper}) \\\hline
Unrestricted RBG            & PSPACE-complete                                & EXPTIME-complete                                    & PSPACE-complete \\\hline
$k$-straight RBG ($k \ge 1$)& $\O((|\mathcal{R}|\cdot|\mathcal{S}|)^{k+1})$  & EXPTIME-complete                                    & PSPACE-complete \\\hline
GDL                         & EXPTIME-complete                               & 2-EXPTIME-complete                                  & EXPSPACE-complete \\\hline
\end{tabular}
\end{table*}

\noindent\textbf{Straight RBG}.
We define subclasses of RBG that exhibit better computational properties.
By condition~(1), the number of modifiers during every play is bounded, thus it is also bounded during a single move, i.e., between switches.
The latter is our measure of the complexity of a description.

Given a language $L$, let $\mathrm{mseq}(L)$ be the set of all factors (substrings) of the words in $L$ that do not contain a switch.
We say that an RBG description is \emph{$k$-straight} if the straightness of $\mathrm{mseq}(\mathrm{app}(\mit{InitialState},\mathcal{L}(\mit{Rules})))$ is at most $k$.
When there are no patterns, the straightness is just the maximum length of a legal move (not counting the final switch).
For example, the description in Fig.~\ref{fig:breakthrough} is $3$-straight but not $2$-straight, because there are three modifiers in lines~17--18 and no more than three modifiers can be applied between switches.
Straightness is difficult to compute exactly, but in many cases a reasonable upper bound can be provided.
In a wide subclass of descriptions, the straightness is bounded solely by the rules, independently on the game states, and the bound can be easily computed.

\noindent\textbf{Complexity}.
We consider three representative decision problems, which are important for agents and game managers and characterize the complexity of RBG.
The input is an abstract RBG description, but for more precise results, we split it into the generalized rules $\mathcal{R} = (\mit{Players},\mit{Pieces},\mit{Variables},\mit{Rules})$ and a game instance $\mathcal{S} = (\mit{Board},\mit{Bounds},\mit{InitialState})$.
By the lengths $|\mathcal{R}|$ and $|\mathcal{S}|$ we understand the lengths of their straightforward text representations similar to those in LL-RBG.

\smallskip
\noindent{\bf Problem~1.}\namedlabel{pbm:haslegal}{1} \textit{Does the first player have a legal move?}\\
{\bf Problem~2.}\namedlabel{pbm:winning}{2} \textit{Does the first player have a winning strategy?}\\
{\bf Problem~3.}\namedlabel{pbm:proper}{3} \textit{Is the game description proper?}
\smallskip

\noindent Problem~\ref{pbm:haslegal} is a very basic problem that every player and game manager must solve to play the game.
Problem~\ref{pbm:winning} is the classical problem of solving the game; we can assume that winning means getting a larger score than the other players.
Problem~\ref{pbm:proper} is important for game designers and is related to exploring the whole game tree.

A basic reasoning algorithm for RBG is based on a DFS on game states; its careful analysis leads to the following:
\begin{theorem}\label{thm:computing-moves2}
For a given $k$-straight description $(k \ge 1)$, the set of all legal moves can be found in $\O((|\mathcal{R}|\cdot|\mathcal{S}|)^k)$ time and in $\O(k(|\mathcal{R}|\cdot|\mathcal{S}|))$ space.
\end{theorem}

\noindent Our results are summarized in Tab.~\ref{tab:complexity}.
For our polynomial result, we made a simplifying assumption that arithmetic operations on variables bounded by $\mit{Bounds}$ can be performed in constant time, which is a practical assumption unless variables are exploited.
The model where both $\mathcal{R}$ and $\mathcal{S}$ are given is the common situation occurring in GGP, where a player sees the game for the first time and must be able to play it shortly after that.
In another scenario, a player knows the rules before and can e.g.\ spend some time on analyzing and preprocessing, possibly even with human aid; then the $\mathcal{R}$ can be considered fixed, and only $\mathcal{S}$ is given.
In this second case, all our hardness results hold as well.

\begin{table*}[htb]\small\renewcommand{\arraystretch}{1.3}
\newcommand{\rowt}[1]{\multirow{2}{*}{#1}}
\newcommand{\rowtt}[1]{\multirow{3}{*}{#1}}
\newcommand{\col}[1]{\multicolumn{1}{c|}{#1}}
\newcommand{\colL}[1]{\multicolumn{1}{c||}{#1}}
\newcommand{\colc}[1]{\multicolumn{1}{|c|}{#1}}
\newcommand{\colt}[1]{\multicolumn{2}{c|}{#1}}
\newcommand{\coltL}[1]{\multicolumn{2}{c||}{#1}}
\newcommand{\coltt}[1]{\multicolumn{3}{c|}{#1}}
\newcommand{\colttt}[1]{\multicolumn{4}{c|}{#1}}
\newcommand{\fstr}{\colL{\rowt{\footnotesize\bf Straightness}}}
\newcommand{\fa}{*}
\newcommand{\fb}{\textdagger}
\caption{The average number of nodes per second for a selection of classical games. The tests were done on a single core of Intel(R) Core(TM) i7-4790 @3.60GHz with 16GB RAM, spending at least ${\sim}10$ min.\ per test.}
\label{tab:experiments}
\begin{center}\begin{tabular}{|l|r||r|r|r|r||r|r|r|}\hline
\colc{\rowt{\bf Game}}&\fstr   &\colt{\bf RBG Compiler}  &\coltL{\bf RBG Interpreter}& \colt{\bf GDL Propnet}  &\bf GDL Prolog\\\cline{3-9}
                      &        &\col{Perft}&\col{Flat MC}&\col{Perft}&\colL{Flat MC} &\col{Perft}&Flat MC& Flat MC      \\\hline
Amazons               &      3 & 7,778,364 &      43,263 & 6,604,412 &        23,371 &    78,680 &         242 &                  13 \\\hline
Arimaa                &$\le 52$&  403,30\fa&          18 &  21,940\fa&             2 & \coltt{\emph{not available}}\\\hline
Breakthrough          &      3 & 9,538,135 &   1,285,315 & 5,113,725 &       371,164 &   589,111 &     175,888 &               4,691 \\\hline
Chess                 &      6 & 2,215,307 &     148,248 &   315,120 &        16,708 &   396,367 &      14,467 &                 120 \\\hline
Chess (without check) &      6 & 6,556,244 &     422,992 & 2,083,811 &        87,281 &   685,546 &      23,625 &               2,702 \\\hline
Connect four          &      2 & 8,892,356 &   3,993,392 & 2,085,062 &     1,024,000 & 3,376,991 &     985,643 &              10,606 \\\hline
Double chess          &      5 & 1,159,707 &      22,095 &   152,144 &         2,249 & \coltt{\emph{not available}}\\\hline
English checkers      &     14 & 3,589,042 &   1,312,813 &   813,439 &       233,519 & 698,829\fb&   225,143\fb&             6,359\fb\\\hline
Go                    &      2 &   557,691 &      66,803 &   137,499 &        17,565 & \coltt{\emph{not available}}\\\hline
Hex (9x9)             &      3 &10,289,555 &   1,048,963 & 5,962,637 &       444,243 &   366,410 &      35,682 &               1,263 \\\hline
International checkers&$\le 44$&   924,288 &     208,749 &   118,227 &        26,754 & \coltt{\emph{not available}}\\\hline
Reversi               &      7 & 2,505,279 &     526,580 &   263,945 &        93,601 &   172,756 &      22,994 & 0                   \\\hline
\end{tabular}\end{center}
\fa\ Arimaa's perft was computed starting from a fixed chess-like position to skip the initial setup.\\
\fb\ English checkers in GDL splits capturing rides, allowing only a single capture per turn (no more accurate version is available).\\
\end{table*}

In conclusion, the complexity of RBG seems to be a good choice, especially for board games. Efficient (polynomial) embeddings are possible because in most popular games these problems can be solved polynomially in the size of the representation (i.e., board). However, there are exceptions, e.g.\ international checkers, where deciding if a given move is legal is coNP-complete, thus polynomial complexity of this task would be insufficient for a concise GGP language.

For a comparison, in Tab.~\ref{tab:complexity} we included GDL \cite{Saffidine2014TheGame} (\emph{bounded GDL}, the version that is used in practice).

\section{Experiments}

We have implemented a computational package for RBG~\cite{Kowalski2018RBGsource}:
a \emph{parser} (of HL-RBG and LL-RBG), an \emph{interpreter} that performs reasoning, a \emph{compiler} that generates a reasoner with a uniform interface, and a \emph{game manager} with example simple players.

To test the efficiency of reasoning, we used two common settings: computing the whole game tree to a fixed depth (\emph{Perft}), and choosing a legal move randomly and uniformly (\emph{flat Monte-Carlo}).
They represent both extremal cases: MC is dominated by computing legal moves, as for each visited node we compute all of them, and in Perft the number of computed nodes (plus one) is equal to the number of applied moves.
The results of our experiments are shown in Table~\ref{tab:experiments}.

For a comparison with GDL, we used the fastest available game implementations (when existing), which we took from \cite{Schreiber2016Games}.
We tested on the same hardware one of the most efficient GDL reasoner based on a propositional network (\emph{propnet}) from \cite{Sironi2016Optimizing}, together with a traditional Prolog reasoner \cite{Schiffel2015GGPServer}.
The idea behind propnets is to express the dynamic of the game in a logic circuit, which can be processed using fast low-level instructions.
Although in general, the propositional networks are the fastest known reasoning algorithm for GDL, initializing a propnet can be heavily time and memory consuming, which is troublesome for large games \cite{Schiffel2014Efficiency}.
For this reason, games exceeding some complexity level are unplayable via the propnet approach.

Summarizing, for simple games both RBG and high-speed GDL reasoners can achieve similar performance (e.g.\ connect four), but a larger size or more difficult rules (e.g.\ amazons, hex, reversi) usually start to make a significant difference.
An exception of chess is mostly caused by the different logic in implementation: RBG uses a general rule to test a check, while the GDL version is optimized with an extensive case analysis (cf.\ chess without check game).
Also, our RBG reasoners are less memory consuming and require smaller initialized time comparing to the propnet.
Finally, RBG allows playing even more complex games and using more accurate rules, which seems to be impossible in GDL at all.
Particularly difficult games for GDL reasoning are those with many moves and moves consisting of multiple steps.
For example, games like go and international checkers, concisely expressible and efficient in RBG, were never encoded in GDL, as they are very difficult to implement and a possible implementation would be likely unplayable.

Finally, we mention the most complex game we have implemented: arimaa -- a game designed to be difficult for computers (it took 12 years to beat the arimaa AI challenge \cite{Wu2015Designing}). The small result for MC comes as a consequence of that computing a single node is equivalent to compute all legal moves for it, which are roughly $200,000$ (as flat MC does not merge the moves that yield the same game state).
In fact, together with the perft result, this shows that arimaa is quite playable, especially if one will be computing selectively a subset of the moves, which is easy in RBG since it defines intermediate game states.

\section{Conclusions}

Being more suitable for particular purposes, RBG fills certain gaps in the existing GGP languages and opens new possibilities.
In particular, RBG allows what was not possible with GDL so far.
Applications lie in all areas of GGP like developing universal agents, learning, procedural content generation, and game analysis. 
Also, developing translations between RBG and GDL is an interesting direction for future research.
RBG has the following advantages:

\noindent\emph{Efficient reasoning}. Essentially, all game playing approaches (min-max, Monte-Carlo search, reinforcement learning, evolutionary algorithms, etc.) require fast reasoning engines.

RBG allows very effective reasoning and is a step toward achieving a high-level universal language with similar performance to game-specific reasoners.
In fact, the natural reasoning algorithm for RBG is a generalization of methods commonly used in game-specific reasoners (e.g.\ modifying game states with elementary actions and backtracking).

\noindent\emph{Complex games}.
RBG allows effective processing complex games and games with potentially large branching factor (e.g.\ amazons, arimaa, non-simplified checkers, double chess, go).
In the existing languages, reasoning for such games is extremely expensive and, in many cases, the game cannot be played at all (except for languages like Ludi or Metagame, which contain dedicated features to define e.g.\ longest ride in checkers, but cannot encode them through universal elements).
Except for amazons and checkers, the above-mentioned games were even not encoded in any GGP language.
Therefore, RBG makes very complex games accessible for GGP through a universal approach.

\noindent\emph{Board games structure}.
It is not surprising that a more specialized language makes developing knowledge-based approaches easier.
RBG allows defining the game structure in a natural way, especially for board games.
It has the concept of the board and pieces, and the rules form a graph (finite automaton).
These are great hints for analyzing the game and an easy source of heuristics, which also make it convenient for procedural content generation.

\noindent\emph{Natural representations}.
Elementary RBG actions that players can perform on the game state directly correspond to the usual understanding of game rules by non-experts, e.g.\ removing and adding pieces, moving on the board to another square.
We can also encode in a natural way advanced rules that usually are difficult to formalize (e.g.\ chess en~passant, capturing ride in checkers).
In our experience, encoding rules in RBG is easier than in GDL, especially for complex games, and the descriptions are considerably shorter.

\noindent\emph{Generalized games}.
Encoding any game in RBG directly separates the generalized rules from a particular instance (board).
In the existing languages, such a separation requires a special effort and is possible only to some extent (e.g.\ in GDL, if the rules are fixed and only the initial game state is the input, the number of possible legal moves depends polynomially on the input size, thus the full rules of checkers cannot be encoded in this way, since we can have an exponential number of moves).
Hence, RBG can open a new investigation setting in GGP, where a player can learn the rules in advance and then must play on any given instance.

\section{Acknowledgements}

This work was supported by the National Science Centre, Poland under project number
2017/25/B/ST6/01920.

\bibliographystyle{aaai}
\bibliography{bibliography}

\begin{thebibliography}{}

\bibitem[\protect\citeauthoryear{Abiteboul, Hull, and
  Vianu}{1995}]{Abiteboul1995Foundations}
Abiteboul, S.; Hull, R.; and Vianu, V., eds.
\newblock 1995.
\newblock {\em {Foundations of Databases: The Logical Level}}.
\newblock Addison-Wesley Longman Publishing Co., Inc., 1st edition.

\bibitem[\protect\citeauthoryear{Bellemare \bgroup et al\mbox.\egroup
  }{2013}]{Bellemare2013TheArcade}
Bellemare, M.~G.; Naddaf, Y.; Veness, J.; and Bowling, M.
\newblock 2013.
\newblock {The Arcade Learning Environment: An Evaluation Platform for General
  Agents}.
\newblock {\em Journal of Artificial Intelligence Research} 47:253--279.

\bibitem[\protect\citeauthoryear{Bj\"{o}rnsson}{2012}]{Bjornsson2012Learning}
Bj\"{o}rnsson, Y.
\newblock 2012.
\newblock {Learning Rules of Simplified Boardgames by Observing}.
\newblock In {\em ECAI}, volume 242 of {\em FAIA}. IOS Press.
\newblock  175--180.

\bibitem[\protect\citeauthoryear{Browne and
  Maire}{2010}]{Browne2010Evolutionary}
Browne, C., and Maire, F.
\newblock 2010.
\newblock {Evolutionary game design}.
\newblock {\em IEEE Transactions on Computational Intelligence and AI in Games}
  2(1):1--16.

\bibitem[\protect\citeauthoryear{Brzozowski}{1964}]{Brz64}
Brzozowski, J.~A.
\newblock 1964.
\newblock Derivatives of regular expressions.
\newblock {\em J. ACM} 11(4):481--494.

\bibitem[\protect\citeauthoryear{Campbell, Hoane, and
  Hsu}{2002}]{Campbell2002Deep}
Campbell, M.; Hoane, A.~J.; and Hsu, F.
\newblock 2002.
\newblock {Deep Blue}.
\newblock {\em Artificial intelligence} 134(1):57--83.

\bibitem[\protect\citeauthoryear{Dickins}{1971}]{Dickins1971AGuide}
Dickins, A.
\newblock 1971.
\newblock {\em A Guide to Fairy Chess}.
\newblock Dover.

\bibitem[\protect\citeauthoryear{Finnsson and
  Bj\"{o}rnsson}{2008}]{Finnsson2008Simulation}
Finnsson, H., and Bj\"{o}rnsson, Y.
\newblock 2008.
\newblock {Simulation-based Approach to General Game Playing}.
\newblock In {\em AAAI Conference on Artificial Intelligence}.

\bibitem[\protect\citeauthoryear{Finnsson and
  Bj\"{o}rnsson}{2010}]{Finnsson2010Learning}
Finnsson, H., and Bj\"{o}rnsson, Y.
\newblock 2010.
\newblock {Learning Simulation Control in General Game Playing Agents}.
\newblock In {\em AAAI Conference on Artificial Intelligence},  954--959.

\bibitem[\protect\citeauthoryear{Fraenkel and
  Lichtenstein}{1981}]{Fraenkel1981Computing}
Fraenkel, A.~S., and Lichtenstein, D.
\newblock 1981.
\newblock {Computing a perfect strategy for n$\times$ n chess requires time
  exponential in n}.
\newblock {\em Journal of Combinatorial Theory, Series A} 31(2):199--214.

\bibitem[\protect\citeauthoryear{Genesereth and
  Thielscher}{2014}]{Genesereth2014General}
Genesereth, M., and Thielscher, M.
\newblock 2014.
\newblock {\em {General Game Playing}}.
\newblock Morgan \& Claypool.

\bibitem[\protect\citeauthoryear{Genesereth, Love, and
  Pell}{2005}]{Genesereth2005General}
Genesereth, M.; Love, N.; and Pell, B.
\newblock 2005.
\newblock {General Game Playing: Overview of the AAAI Competition}.
\newblock {\em AI Magazine} 26:62--72.

\bibitem[\protect\citeauthoryear{Kaiser and
  Stafiniak}{2011}]{Kaiser2011FirstOrder}
Kaiser, L., and Stafiniak, L.
\newblock 2011.
\newblock {First-Order Logic with Counting for General Game Playing}.
\newblock In {\em AAAI Conference on Artificial Intelligence}.

\bibitem[\protect\citeauthoryear{Kowalski \bgroup et al\mbox.\egroup
  }{2018a}]{Kowalski2018RBGextended}
Kowalski, J.; Mika, M.; Sutowicz, J.; and Szyku{\l}a, M.
\newblock 2018a.
\newblock {Regular Boardgames}.
\newblock arXiv:1706.02462 [cs.AI].

\bibitem[\protect\citeauthoryear{Kowalski \bgroup et al\mbox.\egroup
  }{2018b}]{Kowalski2018RBGsource}
Kowalski, J.; Mika, M.; Sutowicz, J.; and Szyku{\l}a, M.
\newblock 2018b.
\newblock {Regular Boardgames -- source code}.
\newblock \url{https://github.com/marekesz/rbg1.0/}.

\bibitem[\protect\citeauthoryear{Love \bgroup et al\mbox.\egroup
  }{2006}]{Love2008General}
Love, N.; Hinrichs, T.; Haley, D.; Schkufza, E.; and Genesereth, M.
\newblock 2006.
\newblock {General Game Playing: Game Description Language Specification}.
\newblock Technical report, Stanford Logic Group.

\bibitem[\protect\citeauthoryear{Newell, Shaw, and
  Simon}{1959}]{Newell1959Report}
Newell, A.; Shaw, J.~C.; and Simon, H.~A.
\newblock 1959.
\newblock {Report on a general problem solving program}.
\newblock In {\em IFIP congress}, volume 256, ~64.

\bibitem[\protect\citeauthoryear{Pell}{1992}]{Pell1992METAGAME}
Pell, B.
\newblock 1992.
\newblock {METAGAME in Symmetric Chess-Like Games}.
\newblock In {\em Heuristic Programming in Artificial Intelligence: The Third
  Computer Olympiad.}

\bibitem[\protect\citeauthoryear{Perez \bgroup et al\mbox.\egroup
  }{2016}]{Perez2016General}
Perez, D.; Samothrakis, S.; Togelius, J.; Schaul, T.; and Lucas, S.~M.
\newblock 2016.
\newblock {General Video Game AI: Competition, Challenges and Opportunities}.
\newblock In {\em AAAI Conference on Artificial Intelligence},  4335--4337.

\bibitem[\protect\citeauthoryear{Pitrat}{1968}]{Pitrat1968Realization}
Pitrat, J.
\newblock 1968.
\newblock {Realization of a general game-playing program}.
\newblock In {\em IFIP Congress},  1570--1574.

\bibitem[\protect\citeauthoryear{Rasmusen}{2007}]{Rasmusen1994Games}
Rasmusen, E.
\newblock 2007.
\newblock {\em {Games and Information: An Introduction to Game Theory}}.
\newblock Blackwell, 4th ed.

\bibitem[\protect\citeauthoryear{Saffidine}{2014}]{Saffidine2014TheGame}
Saffidine, A.
\newblock 2014.
\newblock {The Game Description Language Is Turing Complete}.
\newblock {\em IEEE Transactions on Computational Intelligence and AI in Games}
  6(4):320--324.

\bibitem[\protect\citeauthoryear{Schaeffer \bgroup et al\mbox.\egroup
  }{2007}]{Schaeffer2007Checkers}
Schaeffer, J.; Burch, N.; Bj{\"o}rnsson, Y.; Kishimoto, A.; M{\"u}ller, M.;
  Lake, R.; Lu, P.; and Sutphen, S.
\newblock 2007.
\newblock Checkers is solved.
\newblock {\em Science} 317(5844):1518--1522.

\bibitem[\protect\citeauthoryear{Schaul}{2013}]{Schaul2013AVideo}
Schaul, T.
\newblock 2013.
\newblock {A video game description language for model-based or interactive
  learning}.
\newblock In {\em IEEE Conference on Computational Intelligence and Games},
  1--8.

\bibitem[\protect\citeauthoryear{Schiffel and
  Bj\"{o}rnsson}{2014}]{Schiffel2014Efficiency}
Schiffel, S., and Bj\"{o}rnsson, Y.
\newblock 2014.
\newblock {Efficiency of GDL Reasoners}.
\newblock {\em IEEE Transactions on Computational Intelligence and AI in Games}
  6(4):343--354.

\bibitem[\protect\citeauthoryear{Schiffel and
  Thielscher}{2014}]{Schiffel2014Representing}
Schiffel, S., and Thielscher, M.
\newblock 2014.
\newblock {Representing and Reasoning About the Rules of General Games With
  Imperfect Information}.
\newblock {\em Journal of Artificial Intelligence Research} 49:171--206.

\bibitem[\protect\citeauthoryear{Schiffel}{2015}]{Schiffel2015GGPServer}
Schiffel, S.
\newblock 2015.
\newblock {General Game Playing}.
\newblock \url{http://www.general-game-playing.de/downloads.html}.

\bibitem[\protect\citeauthoryear{Schreiber}{2016}]{Schreiber2016Games}
Schreiber, S.
\newblock 2016.
\newblock {Games -- base repository}.
\newblock \url{http://games.ggp.org/base/}.

\bibitem[\protect\citeauthoryear{Schreiber}{2017}]{Schreiber2017GGPBase}
Schreiber, S.
\newblock 2017.
\newblock {The General Game Playing Base Package}.
\newblock \url{https://github.com/ggp-org/}.

\bibitem[\protect\citeauthoryear{Silver \bgroup et al\mbox.\egroup
  }{2016}]{Silver2016Mastering}
Silver, D.; Huang, A.; Maddison, C.~J.; Guez, A.; Sifre, L.; van~den Driessche,
  G.; Schrittwieser, J.; Antonoglou, I.; Panneershelvam, V.; Lanctot, M.;
  Dieleman, S.; Grewe, D.; Nham, J.; Kalchbrenner, N.; Sutskever, I.;
  Lillicrap, T.; Leach, M.; Kavukcuoglu, K.; Graepel, T.; and Hassabis, D.
\newblock 2016.
\newblock {Mastering the game of Go with deep neural networks and tree search}.
\newblock {\em Nature} 529:484--503.

\bibitem[\protect\citeauthoryear{Sironi and
  Winands}{2017}]{Sironi2016Optimizing}
Sironi, C.~F., and Winands, M. H.~M.
\newblock 2017.
\newblock {Optimizing Propositional Networks}.
\newblock In {\em Computer Games}. Springer.
\newblock  133--151.

\bibitem[\protect\citeauthoryear{Stockmeyer and
  Chandra}{1979}]{Stockmeyer1979Provably}
Stockmeyer, L.~J., and Chandra, A.~K.
\newblock 1979.
\newblock {Provably difficult combinatorial games}.
\newblock {\em SIAM Journal on Computing} 8(2):151--174.

\bibitem[\protect\citeauthoryear{{\'S}wiechowski \bgroup et al\mbox.\egroup
  }{2015}]{Swiechowski2015Recent}
{\'S}wiechowski, M.; Park, H.; Ma{\'n}dziuk, J.; and Kim, K.
\newblock 2015.
\newblock {Recent Advances in General Game Playing}.
\newblock {\em The Scientific World Journal} 2015.

\bibitem[\protect\citeauthoryear{Thielscher}{2011}]{Thielscher2011TheGeneral}
Thielscher, M.
\newblock 2011.
\newblock {The General Game Playing Description Language is Universal}.
\newblock In {\em International Joint Conference on Artificial Intelligence},
  1107--1112.

\bibitem[\protect\citeauthoryear{Thompson}{1968}]{Thompson1968ProgrammingTechniques}
Thompson, K.
\newblock 1968.
\newblock Programming techniques: Regular expression search algorithm.
\newblock {\em Commun. ACM} 11(6):419--422.

\bibitem[\protect\citeauthoryear{Wu}{2015}]{Wu2015Designing}
Wu, D.
\newblock 2015.
\newblock Designing a winning arimaa program.
\newblock {\em ICGA Journal} 38(1):19--40.

\end{thebibliography}
\newpage\onecolumn
\section{\LARGE Appendix}

The appendix contains the extended experimental results, ommited proofs, and the technical specification of the RBG language.

\section{Extended experimental results}

In this section, we present the extended results of our RBG interpreter and compiler \cite{Kowalski2018RBGsource} in comparison to GDL reasoning engines: propositional networks-based \cite{Sironi2016Optimizing}, a traditional Prolog reasoner (ECLiPSe Prolog) \cite{Schiffel2015GGPServer}, and a GGP prover from \cite{Schreiber2017GGPBase}.
We have implemented in RBG a selection of classical games.

We have tested the efficiency of reasoning in the following two scenarios:
\begin{itemize}
\item \textbf{Perft}: We compute the whole game tree to a fixed depth, starting from the initial state. This is the common performance and debugging test applied for game-specific reasoners, particularly popular for chess programs.
\item \textbf{Flat Monte Carlo}: We perform random playouts from the initial state to the leaves. For every visited state we compute all legal moves, choose one of them uniformly at random, and apply it. This test is the most popular setting in GGP and in can be found in every GDL benchmark.
\end{itemize}
Flat MC is dominated by computing legal moves, whereas perft has a balance between computing moves and states.
In particular, games with a large branching factor get significantly smaller results in MC.

The results for perft are presented in Table~\ref{tab:experiments-perft}, and for flat Monte Carlo in Table~\ref{tab:experiments-mc}.
All the tests were done on a single core of Intel(R) Core(TM) i7-4790 @3.60GHz of a computer with 16GB RAM.

For some of the reasoners there are additional properties playing an important role: the initialization time, and the memory required for the computation.
These values are presented in Table~\ref{tab:experiments-other}.

For GDL, we took the fastest available implementations with the most accurate rules of tested games (if existing), downloaded from~\cite{Schreiber2016Games}.
The particular implementations that we tested are as follows:
amazons -- \texttt{amazons.kif},
amazons (split) -- \texttt{amazons\_10x10.kif},
breakthrough -- \texttt{breakthrough.kif},
English checkers -- \texttt{englishDraughts.kif},
chess -- \texttt{chess\_200.kif},
chess (without check) -- \texttt{speedChess.kif},
connect four -- \texttt{connectFour.kif},
gomoku -- \texttt{gomoku\_15x15.kif},
hex (9x9) -- \texttt{hex.kif},
reversi -- \texttt{reversi.kif},
tic-tac-toe -- \texttt{ticTacToe.kif}.

For connect four, we fixed the GDL description to describe the original 7x6 version.
English checkers in GDL splits capturing rides, allowing only a single capture per turn (no more accurate version is available).
There also exists a simplified Chinese checkers version in GDL; however, it is so distant from the official rules that we did not consider it.

\newpage
\begin{table*}[!htb]\small\renewcommand{\arraystretch}{1.2}
\newcommand{\rowt}[1]{\multirow{2}{*}{#1}}
\newcommand{\rowtt}[1]{\multirow{3}{*}{#1}}
\newcommand{\col}[1]{\multicolumn{1}{c|}{#1}}
\newcommand{\colL}[1]{\multicolumn{1}{c||}{#1}}
\newcommand{\colc}[1]{\multicolumn{1}{|c|}{#1}}
\newcommand{\colcL}[1]{\multicolumn{1}{|c||}{#1}}
\newcommand{\colt}[1]{\multicolumn{2}{c|}{#1}}
\newcommand{\coltL}[1]{\multicolumn{2}{c||}{#1}}
\newcommand{\coltt}[1]{\multicolumn{3}{c|}{#1}}
\newcommand{\colttt}[1]{\multicolumn{4}{c|}{#1}}
\newcommand{\fa}{*}
\newcommand{\fb}{\textdagger}
\newcommand{\fc}{\text{$\ddagger$}}
\caption{Efficiency comparison of RBG and GDL resoners for the \textbf{perft} test. The values are gathered spending ${\sim}10$ min.\ per test, and presented as the average number of computed nodes per second.}\label{tab:experiments-perft}
\begin{center}\begin{tabular}{|l||r|r||r|r|}\hline
\colcL{\rowt{\bf Game}}    &\coltL{\bf RBG}                  &\colt{\bf GDL}            \\\cline{2-5}
                          &\col{Compiler}&\colL{Interpreter}&\col{Propnet}&\col{Prover}\\\hline
Amazons                   &    7,778,364 &       6,604,412 &      78,680 &        487 \\\hline          
Amazons (split)           &    9,825,526 &       6,609,850 &   1,154,050 &      5,762 \\\hline
Arimaa                    &    403,304\fa&        21,940\fa& \colt{\emph{not available}}               \\\hline
Breakthrough              &    9,538,135 &       5,113,725 &     589,111 &      2,914 \\\hline
Chess                     &    2,215,307 &         315,120 &     396,367 &        869 \\\hline
Chess (without check)     &    6,556,244 &       2,083,811 &     685,546 &      1,285 \\\hline
Chinese checkers 6-players&    7,371,902 &       2,111,145 & \colt{\emph{not available}}               \\\hline
Connect four              &    8,892,356 &       2,085,062 &   3,376,991 &      5,979 \\\hline
Double chess              &    1,159,707 &         152,144 & \colt{\emph{not available}}               \\\hline
English checkers          &    3,589,042 &         813,439 &     698,829 &      1,034 \\\hline
Go                        &      557,691 &         137,499 & \colt{\emph{not available}}               \\\hline
Gomoku                    &    7,104,911 &       2,285,159 &   4,007,001 &     31,141 \\\hline
Hex (11x11)               &    9,196,438 &       6,097,961 & \colt{\emph{not available}}               \\\hline
Hex (9x9)                 &   10,289,555 &       5,962,637 &     366,410 &     10,928 \\\hline
International checkers    &      924,288 &         118,227 & \colt{\emph{not available}}               \\\hline
Paper soccer              &    3,105,158 &         781,524 & \colt{\emph{not available}}               \\\hline
Reversi                   &    2,505,279 &         263,945 &     172,756 &        213 \\\hline
The mill game             &    5,992,142 &       1,402,369 & \colt{\emph{not available}}               \\\hline
Tic-tac-toe               &  6,321,218\fb&     2,448,739\fb&   390,530\fb&      5,538 \\\hline
\end{tabular}\end{center}\begin{flushleft}
\fa\ Arimaa's perft was computed starting from a fixed chess-like position to skip the initial setup.\\
\fb\ For tic-tac-toe, the whole game tree is computed in about a second, thus the test could not last enough time to provide reliable results.\\
\end{flushleft}\end{table*}

\newpage

\begin{table*}[!htb]\small\renewcommand{\arraystretch}{1.2}
\newcommand{\rowt}[1]{\multirow{2}{*}{#1}}
\newcommand{\rowtt}[1]{\multirow{3}{*}{#1}}
\newcommand{\col}[1]{\multicolumn{1}{c|}{#1}}
\newcommand{\colL}[1]{\multicolumn{1}{c||}{#1}}
\newcommand{\colc}[1]{\multicolumn{1}{|c|}{#1}}
\newcommand{\colcL}[1]{\multicolumn{1}{|c||}{#1}}
\newcommand{\colt}[1]{\multicolumn{2}{c|}{#1}}
\newcommand{\coltL}[1]{\multicolumn{2}{c||}{#1}}
\newcommand{\coltt}[1]{\multicolumn{3}{c|}{#1}}
\newcommand{\colttt}[1]{\multicolumn{4}{c|}{#1}}
\newcommand{\fa}{*}
\caption{Efficiency comparison of RBG and GDL reasoners for the \textbf{flat Monte Carlo} test. The values are gathered spending ${\sim}10$ min.\ per test, and presented as the average number of computed nodes per second.}\label{tab:experiments-mc}
\begin{center}\begin{tabular}{|l||r|r||r|r|r|}\hline
\colcL{\rowt{\bf Game}}    &\coltL{\bf RBG}                  &\coltt{\bf GDL}                        \\\cline{2-6}
                          &\col{Compiler}&\colL{Interpreter}&\col{Propnet}&\col{Prolog}&\col{Prover}\\\hline
Amazons                   &       43,263 &          23,371 &         242 &         13 &          2 \\\hline          
Amazons (split)           &    1,176,877 &         288,201 &      50,764 &      3,721 &        366 \\\hline
Arimaa                    &           18 &               2 & \coltt{\emph{not available}}                           \\\hline
Breakthrough              &    1,285,315 &         371,164 &     175,888 &      4,691 &        783 \\\hline
Chess                     &      148,248 &          16,708 &      14,467 &        120 &         64 \\\hline
Chess (without check)     &      422,992 &          87,281 &      23,625 &      2,702 &        213 \\\hline
Chinese checkers 6-players&      282,821 &          53,186 & \coltt{\emph{not available}}                           \\\hline
Connect four              &    3,993,392 &       1,024,000 &     985,643 &     10,606 &      1,829 \\\hline
Double chess              &       22,095 &           2,249 & \coltt{\emph{not available}}                           \\\hline
English checkers          &    1,312,813 &         233,519 &     225,143 &      6,359 &        873 \\\hline
Go                        &       66,803 &          17,565 & \coltt{\emph{not available}}                           \\\hline
Gomoku                    &      267,836 &         176,417 &     122,424 &        480 &         95 \\\hline
Hex (11x11)               &      735,901 &         329,984 & \coltt{\emph{not available}}                           \\\hline
Hex (9x9)                 &    1,048,963 &         444,243 &      35,682 &      1,263 &        275 \\\hline
International checkers    &      208,749 &          26,754 & \coltt{\emph{not available}}                           \\\hline
Paper soccer              &       20,195 &           2,742 & \coltt{\emph{not available}}                           \\\hline
Reversi                   &      526,580 &          93,601 &      22,994 &        0\fa&         37 \\\hline
The mill game             &      487,958 &          64,003 & \coltt{\emph{not available}}                           \\\hline
Tic-tac-toe               &    4,098,300 &       1,725,347 &     774,852 &     26,935 &      5,813 \\\hline
\end{tabular}\end{center}
\begin{flushleft}
\fa For reversi, Prolog could not complete even one playout within a reasonable time.\\
\end{flushleft}\end{table*}

\begin{table*}[!htb]\centering\small\renewcommand{\arraystretch}{1.2}
\newcommand{\rowt}[1]{\multirow{2}{*}{#1}}
\newcommand{\rowtt}[1]{\multirow{3}{*}{#1}}
\newcommand{\col}[1]{\multicolumn{1}{c|}{#1}}
\newcommand{\colL}[1]{\multicolumn{1}{c||}{#1}}
\newcommand{\colc}[1]{\multicolumn{1}{|c|}{#1}}
\newcommand{\colcL}[1]{\multicolumn{1}{|c||}{#1}}
\newcommand{\colt}[1]{\multicolumn{2}{c|}{#1}}
\newcommand{\coltL}[1]{\multicolumn{2}{c||}{#1}}
\newcommand{\coltt}[1]{\multicolumn{3}{c|}{#1}}
\newcommand{\colttt}[1]{\multicolumn{4}{c|}{#1}}
\caption{Comparison of the resources used. The initialization times for RBG compiler and GDL Propnet are given in seconds. The maximal memory usage is given in MB of RAM.
The RBG interpreter, Prolog, and Prover do not require any significant time for initialization, and the memory for the two latter also should not be an issue.
In the case of the RBG compiler, the higher amount of memory is required only for the compilation (which has a large constant overhead), whereas during the test it remains small (do not even exceed $60$ MB, which is met for arimaa).}\label{tab:experiments-other}
\begin{tabular}{|l||r|r||r|r|r|}\hline
\colcL{\rowtt{\bf Game}}    &\coltL{\bf Initialization time}&\coltt{\bf Memory}                            \\\cline{2-6}
                           &\col{\bf RBG} &\colL{\bf GDL}&\colt{\bf RBG}                   &\col{\bf GDL} \\\cline{2-6}
                           &\col{Compiler}&\colL{Propnet}&\col{Compiler} &\col{Interpreter}&\col{Propnet} \\\hline
Amazons                    &         4.76 &      194.41 &           128 &               8 &        3,043 \\\hline          
Amazons (split)            &         3.09 &        8.71 &           120 &               8 &          340 \\\hline
Arimaa                     &        73.04 & \colL{\emph{not available}}   &           916 &             369 & \col{\emph{not available}}    \\\hline
Breakthrough               &         2.03 &        1.74 &            97 &               6 &          150 \\\hline
Chess                      &        11.65 &       13.17 &           282 &              35 &          349 \\\hline
Chess (without check)      &         5.74 &        6.70 &           163 &              13 &          252 \\\hline
Chinese checkers 6-players &         5.40 & \colL{\emph{not available}}   &           140 &              17 & \col{\emph{not available}}    \\\hline
Connect four               &         2.05 &        0.64 &           103 &               5 &          134 \\\hline
Double chess               &        13.36 & \colL{\emph{not available}}   &           313 &              70 & \col{\emph{not available}}    \\\hline
English checkers           &         3.78 &        2.03 &           131 &              16 &          141 \\\hline
Go                         &         6.72 & \colL{\emph{not available}}   &           195 &             390 & \col{\emph{not available}}    \\\hline
Gomoku                     &         2.78 &        6.66 &           124 &              18 &          270 \\\hline
Hex (11x11)                &         2.40 & \colL{\emph{not available}}   &           106 &               8 & \col{\emph{not available}}    \\\hline
Hex (9x9)                  &         2.25 &        4.45 &           103 &               7 &          255 \\\hline
International checkers     &         8.17 & \colL{\emph{not available}}   &           197 &              35 & \col{\emph{not available}}    \\\hline
Paper soccer               &        11.81 & \colL{\emph{not available}}   &           298 &             168 & \col{\emph{not available}}    \\\hline
Reversi                    &         6.46 &        3.12 &           187 &              12 &          182 \\\hline
The mill game              &        15.62 & \colL{\emph{not available}}   &           274 &              11 & \col{\emph{not available}}    \\\hline
Tic-tac-toe                &         1.91 &        0.44 &            95 &               5 &          131 \\\hline
\end{tabular}
\end{table*}

\newpage\null\newpage
\section{Proofs}

\setcounter{theorem}{3}

\begin{propositionnum}[\ref{pro:unique-derivative}]
For every indexed action $a_i$, the non-empty derivatives of $\mathcal{L}(\hat{E})$ by $u' a_i$ are the same for every word $u'$.
\end{propositionnum}
\begin{proof}
To see this, we can use the Thompson's construction of the NFA with $\varepsilon$-transitions that recognizes $\mathcal{L}(\hat{E})$.
In this construction, for every indexed action $a_i$ there exists a unique edge labeled by $a_i$, and hence there is a unique state $q_i$ that has this incoming edge.
Thus, for an arbitrary word $u'$, the derivative of $\mathcal{L}(\hat{E})$ by $u' a_i$ is either empty, when $u' a_i \notin \mathrm{pref}(\mathcal{L}(\hat{E}))$, or it is the set of all words accepted by the NFA having the initial state $q_i$.
The latter set is completely determined by $a_i$.
\end{proof}

\begin{theoremnum}[\ref{thm:rbg-is-universal}]
RBG is universal for the class of finite deterministic games with full information.
\end{theoremnum}
\begin{proof}
Following \cite{Rasmusen1994Games}, by a finite deterministic game with full information we mean a tuple $(n,T,s_0,\iota,\mathit{out})$, where $n\in \mathbb{N}$ indicates the number of players, $T = (V,E)$ is a finite tree rooted at a vertex $s_0\in V$ and with $V_\mathit{ter} \subseteq V$ as its set of leaves, $\iota\colon V \setminus V_\mathit{ter} \to \{1,\hdots,n\}$ is a function indicating which player has control in the given node, and $\mathit{out}\colon V_\mathit{ter} \to \mathbb{Q}_{\ge 0}^{n}$ is a function defining all players' scores.

We can easily transform the above structure into RBG-conforming elements that define exactly the tree $T$ (except the scores are scaled).
However, this construction does not aim to be optimal, as its size is linear in the size of $T$.

Given a game $(n,T,s_0,\iota,\mathit{out})$, we create an abstract RBG description defining exactly the same game.

We set $\mathit{Players} = \{\mathit{player}_1,\hdots,\mathit{player}_n\}$.
Our description will not need any board nor pieces, thus we can define a trivial one-vertex board.
The only variables we use are the players' scores.
To express rational scores assigned by $\mathit{out}$, it is possible to scale them by some integer $q$, so that all of them become natural numbers. The variable $\mathit{player}_i$ can now be bounded by $\max_{v\in V_\mathit{ter}} \{q\cdot \mathit{out}(v)_i\}$, where $\mathit{out}(v)_i$ is the $i$-th element in the tuple $\mathit{out}(v)$.

We construct the RBG rules inductively over the depth of $T=(V,E)$.
We denote the rules for a game rooted at a vertex $v \in V$ by $\mit{rules}(v)$.
If $v \in V_\mathit{ter}$, then it has assigned some outcome tuple $q \mathit{out}(s_0)=(r_1,\hdots,r_n)$, and represents a game where the players have no moves to perform. We only have to set the outcome and end the play using the keeper:
\begin{align*}
\mit{rules}(v) &= [\$\,\mit{player}_1{=}r_1] \hdots\ [\$\,\mit{player}_n{=}r_n]\,{\doubleto}\,\emptyset.
\end{align*}
If $v$ is not a leaf, then it has some children $v_1,\hdots,v_m$ along with their RBG rules descriptions $\mit{rules}(v_1),\hdots,\mit{rules}(v_m)$. The player performing a move at $v$ is $\iota(v)=i$.
Then we define:
\begin{align*}
\mit{rules}(v) &= {\to}\mit{player}_i\ (\mit{rules}(v_1) + \hdots + \mit{rules}(v_n))
\end{align*}

It is easy to see that every play of the RBG game corresponds to a play (path) of the given game rooted at $v$ and vice versa. The move of $\mit{player}_i$ boils down to the choice of the next vertex, one of the children of $v$.

The rules of the whole game are $\mit{rules}(s_0)$. The correctness of the construction is evident after noting that the rules expression tree looks like $T$ transformed into RBG-conforming elements.
\end{proof}

\subsubsection{Strong straightness}

 As mentioned, the straightness of a description is difficult to compute. To determine it exactly, we need to compute the whole game tree. However, in many cases the straightness is bounded by the rules and a reasonable upper bound can be computed easily. Such descriptions form a special subclass, where the algorithm from Theorem~7 works in polynomial time if the rules are fixed and the board is given.

We formally extend the definition of the application language to languages independent on the semi-state.
For a language $L$, the application language $\mathrm{app}(L)$ is defined recursively by:
$$\mathrm{app}(L) = \mathrm{pref}(L) \cup \{uv \mid u\{? M\}\text{ or }u\{! M\}\text{ is in }\mathrm{pref}(L),\text{ and }v \in \mathrm{app}(M)\}.$$
The only difference with the application language $\mathrm{app}(S,L)$ is that we do not check if the actions are valid, thus it does not depend on $S$.
For example, if $L$ is defined by
$$\mathit{up}\,\{?\,[a]\,\{!\,[b]\}\,[c]\}\,{\doubleto}\,[d]\,{\doubleto},$$
then we have
$$\mathrm{app}(L) = \big\{\mathit{up}\,\{?\,[a]\,\{!\,[b]\}\,[c]\}\,{\doubleto}\,[d]\,{\doubleto}, \mathit{up}\,[a]\,\{!\,[b]\}\,[c], \mathit{up}\,[a]\,[b]\big\}.$$
Then $\mathrm{mseq}(\mathrm{app}(L))$ consists of all factors of the words from $$\big\{\mathit{up}\,\{?\,[a]\,\{!\,[b]\}\,[c]\}, [d], \mathit{up}\,[a]\,\{!\,[b]\}\,[c], \mathit{up}\,[a]\,[b]\big\}.$$

An RBG description is \emph{strongly $k$-straight} if the straightness of $\mathrm{mseq}(\mathrm{app}(\mathcal{L}(\mit{Rules})))$ is at most $k$.
The strong straightness (the smallest $k$ such that the description is strongly $k$-straight) is easily computable in linear time.

Table~\ref{tab:strong-straightness} shows the strong straightness for our implemented games, together with our estimation of the real straightness.

\begin{table*}[htb]\centering\small\renewcommand{\arraystretch}{1.3}
\newcommand{\rowt}[1]{\multirow{2}{*}{#1}}
\newcommand{\rowtt}[1]{\multirow{3}{*}{#1}}
\newcommand{\col}[1]{\multicolumn{1}{c|}{#1}}
\newcommand{\colc}[1]{\multicolumn{1}{|c|}{#1}}
\newcommand{\colt}[1]{\multicolumn{2}{c|}{#1}}
\newcommand{\coltt}[1]{\multicolumn{3}{c|}{#1}}
\caption{The strong straightness of RBG game descriptions.}\label{tab:strong-straightness}
\begin{tabular}{|l|r|r|}\hline
\colc{\bf Game}            & \col{Straightness} & \col{Strong straightness} \\\hline
Arimaa                     &           $\le 52$ & $\infty$                  \\\hline
Amazons                    &                  3 & 3                   \\\hline       
Amazons (split)            &                  3 & 3                   \\\hline
Breakthrough               &                  3 & 3                   \\\hline        
Chess                      &                  6 & 6                   \\\hline
Chess (without check)      &                  6 & 6                   \\\hline
Chinese checkers 6-players &                  8 & 8                   \\\hline
Connect4                   &                  2 & 2                   \\\hline
Double chess               &                  5 & 5                   \\\hline
English checkers           &                 14 & $\infty$            \\\hline
Go (19x19)                 &                  2 & 2                   \\\hline
Gomoku (15x15)             &                  3 & 3                   \\\hline
Hex (9x9)                  &                  3 & 3                   \\\hline
Hex (11x11)                &                  3 & 3                   \\\hline
International checkers     &           $\le 44$ & $\infty$            \\\hline
Paper soccer              &        $\le 2,100$ & $\infty$            \\\hline
Reversi                    &                  7 & $\infty$            \\\hline
The mill game              &                  3 & 3                   \\\hline 
Tic-tac-toe                &                  3 & 3                   \\\hline
\end{tabular}
\end{table*}

\begin{theorem}\label{thm:strongstraightness}
Given a regular expression $E$, the straightness of $\mathrm{mseq}(\mathrm{app}(\mathcal{L}(E)))$ is computable in $\O(|E|)$ time and, if it is finite, it is at most $|E|$, where $|E|$ is the length of the text representation of $E$.
\end{theorem}\label{thm:computing-straightness}
\begin{proof}
The strong straightness can be easily computed recursively for every subexpression of a given regular expression.
We do this by splitting the problem into four subproblems, which together make the final result:
We count the maximum number of modifiers in a word from an application language before the first switch, after the last switch, between two switches with other actions between, and in the words without switches.
The values for an expression are computed using these values recursively for the subexpressions.

For a language $L$, recall $\mathrm{pref}(L)$ and let $\mathrm{fact}(L)$ and $\mathrm{suff}(L)$ be the sets of all prefixes, factors (substrings), and suffixes of the words from $L$, respectively.

We will construct a recursive algorithm which, given an expression $E$, computes the straightness of $\mathrm{mseq}(\mathrm{app}(\mathcal{L}(E)))$.
Note that $\mathrm{app}(\mathcal{L}(E_1 E_2))$ is the language $\mathrm{app}(\mathcal{L}(E_1)) \cup \mathcal{L}(E_1)\cdot \mathrm{app}(\mathcal{L}(E_2))$.
A word with a maximum number of modifiers without switches can be taken from $\mathrm{app}(\mathcal{L}(E_1))$ or from $\mathcal{L}(E_1)\cdot \mathrm{app}(\mathcal{L}(E_2))$.
In the second case, it has the form $w=uv$, where $u$ is a suffix of a word from $\mathcal{L}(E_1)$ and $v$ is from $\mathrm{app}(\mathcal{L}(E_2))$.
We will need four auxiliary functions $\mit{MSuff}$, $\mit{MPref}$, $\mit{MFact}$, and $\mit{MWord}$, that return the maximum number of modifiers in the words from some sets defined by a given regular expression.
Their values are in $\mathbb{N}\cup\{\perp,\infty\}$, where the $\infty$ means that maximum does not exist, and $\perp$ means that the set of which maximum is to be taken is empty.
We define:
\begin{itemize}
\item $\mit{MSuff}(E)$ -- the maximum number of modifiers in the words from $\mathrm{suff}(\mathcal{L}(E)))$ that do not contain any switches.
For example, $\mit{MSuff}\big([a]\{b\}\{? [c][d]\}[e]{\doubleto}[f](\{g\}{+}[h])[i]\big) = 3$, because we have the suffix $\doubleto[f][h][i]$, and no suffix with a larger number modifiers is present.

\item $\mit{MPref}(E)$ -- the maximum number of modifiers in the words from $\mathrm{app}(\mathcal{L}(E))) = \mathrm{pref}(\mathrm{app}(\mathcal{L}(E)))$ that do not contain any switches.
Taking the above example, $\mit{MPref}\big([a]\{b\}\{? [c][d]\}[e]{\doubleto}[f](\{g\}{+}[h])[i]\big) = 4$, because of the prefix $[a]\{b\}[c][d]$ that is present in the application language.

\item $\mit{MFact}(E)$ -- the maximum number of modifiers in the words from $\mathrm{fact}(\mathrm{app}(\mathcal{L}(E)))$ that do not contain any switches.
For example, $\mit{MFact}\big([a]\{b\}\{? [c][d]\}[e]{\doubleto}[f](\{g\}{+}[h])[i]\big) = 4$, which is the same as for $\mit{MPref}$, but $\mit{MFact}\big([a]{\doubleto}[b][c][d]{\doubleto}[e][f]\big)=3$, since there is the factor $[b][c][d]$.

\item $\mit{MWord}(E)$ -- the maximum number of modifiers in the words from $\mathcal{L}(E)$ that do not contain any switches.
Taking the above example, $\mit{MWord}\big([a]\{b\}\{? [c][d]\}[e]{\doubleto}[f](\{g\}{+}[h])[i]\big)=\perp$, because every word from the language has a switch. But $\mit{MWord}\big([a][b]({\doubleto}{+}[c])[d]\big)=4$, because of the word $[a][b][c][d]$.
\end{itemize}
The values $\mit{MSuff}(E)$, $\mit{MPref}(E)$, and $\mit{MWord}(E)$ are respectively the maximum number of modifiers in the words that can begin, end, or simultaneously begin and end a word with the maximum number of modifiers when $E$ is concatenated with the next subexpression.
Obviously, we have $\mit{MWord}(E) \le \mit{MPref}(E) \le \mit{MFact}(E)$ and $\mit{MWord}(E) \le \mit{MSuff}(E) \le \mit{MFact}(E)$.
The value $\mit{MFact}(E)$ is equal to the straightness of $\mathrm{mseq}(\mathrm{app}(\mathcal{L}(E)))$, thus is our result.

Now, we will give formulas for computing all these functions for any expression $E$. To simplify evaluation of maximums and additions, we assume that $\infty > n$, $\perp < n$, $\infty + n = \infty$, $\perp + n = \perp$ and $\infty + \perp = \perp$ for every $n\in\mathbb{N}$.

If $E$ is a modifier action other than a switch:
\begin{align*}
\mit{MSuff}(E) & =\ 1; & \mit{MPref}(E) & =\ 1; \\
\mit{MFact}(E) & =\ 1; & \mit{MWord}(E) & =\ 1.
\end{align*}
If $E$ is a non-modifier action other than a pattern:
\begin{align*}
\mit{MSuff}(E) & =\ 0; & \mit{MPref(E)} & =\ 0; \\
\mit{MFact}(E) & =\ 0; & \mit{MWord(E)} & =\ 0.
\end{align*}
If $E$ is a switch:
\begin{align*}
\mit{MSuff}(E) & =\ 0; & \mit{MPref}(E) & =\ 0; \\
\mit{MFact}(E) & =\ 0; & \mit{MWord}(E) & =\ \perp.
\end{align*}
If $E$ is a pattern $\{? F\}$ or $\{! F\}$, then:
\begin{align*}
\mit{MSuff}(E) &\ = 0; & \mit{MPref}(E) & =\ \mit{MPref}(F); \\
\mit{MFact}(E) &\ =\ \mit{MPref}(F); & \mit{MWord}(E) & =\ 0.
\end{align*}
The three remaining cases cover all the operations. While having many possibilities to cover, they simply represent the behavior of each of the operators. For example, to calculate $\mit{MSuff}(E^*)$ we must take into account all sequences beginning with a switch inside $E$ and see if we can extend them by words from $\mathcal{L}(E)$ without switches; if there is such a word with a positive number of modifiers (i.e., $\mit{MWord}(E) > 0$), then the number of modifiers in the words from $\mathrm{suff}(\mathcal{L}(E^*)))$ is unbounded.
\begin{align*}
\mit{MSuff}(E_1+E_2)&=\max(\mit{MSuff}(E_1),\mit{MSuff}(E_2));\\
\mit{MPref}(E_1+E_2)&=\max(\mit{MPref}(E_1),\mit{MPref}(E_2));\\
\mit{MFact}(E_1+E_2)&=\max(\mit{MFact}(E_1),\mit{MFact}(E_2));\\
\mit{MWord}(E_1+E_2)&=\max(\mit{MWord}(E_1),\mit{MWord}(E_2)).\\
\end{align*}
\begin{align*}
\mit{MSuff}(E_1 E_2)&=\max(\mit{MSuff}(E_1)+\mit{MWord}(E_2),\mit{MSuff}(E_2));\\
\mit{MPref}(E_1 E_2)&=\max(\mit{MPref}(E_1),\mit{MWord}(E_1)+\mit{MPref}(E_2));\\
\mit{MFact}(E_1 E_2)&=\max(\mit{MFact}(E_1),\mit{MSuff}(E_1)+\mit{MPref}(E_2),\mit{MFact}(E_2));\\
\mit{MWord}(E_1 E_2)&=\mit{MWord}(E_1)+\mit{MWord}(E_2).\\
\end{align*}
\begin{align*}
\mit{MSuff}(E^*)&=\begin{cases}
\infty \text{, if } \mit{MWord}(E)>0;\\
\mit{MSuff}(E) \text{, if } \mit{MWord}(E)=0 \text{ or } \mit{MWord}(E)=\perp;
\end{cases}\\
\mit{MPref}(E^*)&=\begin{cases}
\infty \text{, if } \mit{MWord}(E)>0;\\
\mit{MPref}(E) \text{, if } \mit{MWord}(E)=0 \text{ or } \mit{MWord}(E)=\perp;
\end{cases}\\
\mit{MFact}(E^*)&=\begin{cases}
\infty \text{, if } \mit{MWord}(E)>0;\\
\max(\mit{MSuff}(E)+\mit{MPref}(E),\mit{MFact}(E)) \text{, if }  \mit{MWord}(E)=0 \text{ or } \mit{MWord}(E)=\perp;
\end{cases}\\
\mit{MWord}(E^*)&=\begin{cases}
\infty \text{, if } \mit{MWord}(E)>0\\
0 \text{, if } \mit{MWord}(E)=0 \text{ or } \mit{MWord}(E)=\perp.
\end{cases}
\end{align*}

Having these formulas, we can easily see by induction that the strong straightness $(\mit{MFact}(E))$ indeed cannot exceed $|E|$, if only it is finite.
Furthermore, since every formula is computed in constant time provided the results for the subexpressions, the strong straightness can be computed in linear time.
\end{proof}

The following result provides a standard reasoning algorithm for RBG.
This is the algorithm implemented in our package.

\begin{theorem}\label{thm:computing-moves1}
Given a $k$-straight ($k \ge 1$) RBG description $(\mathcal{R},\mathcal{S})$ with pattern depth at most $d$ and a reachable game state $I_r$, all legal moves can be found in $\O((|\mathcal{R}|\cdot|\mathcal{S}|)^{k+1}\cdot|\mathcal{S}|^{d})$ time and $\O(k(|\mathcal{R}|\cdot|\mathcal{S}|))$ space.
\end{theorem}
\begin{proof}
The idea of the algorithm is a DFS on the game states.
To determine legal actions, we use the NFA obtained from the Thompson's construction from the rules.
Every visited game state is obtained by applying some sequence of modifier applications (modifiers together with positions) followed by some non-modifier actions.
In the DFS, we do not repeat game states obtained by the same sequence of modifier applications, i.e., we visit at most $|\mathcal{R}|\cdot|\mathcal{S}|$ game states for the same sequence -- these game states can differ only by the rules index and the position on the board.
To evaluate a pattern, we use the same algorithm recursively.

\noindent\emph{Algorithm}.
Let $E$ be a regular expression and $\Sigma$ be the set of actions occurring in $E$.
Remind that $\hat{E}$ is the regular expression with indexed actions, and let $\hat{\Sigma_E}$ be the set of these indexed actions that occur in $\hat{E}$.

Using the Thompson's construction \cite{Thompson1968ProgrammingTechniques}, for a regular expression $E$, we can build the NFA $\mathcal{N}_E = (Q_E,\hat{\Sigma_E},\eta,\{q_0\},F_E)$ with $\varepsilon$-transitions that recognizes $\mathcal{L}(\hat{E})$, where $Q_E$ is the set of states, $\hat{\Sigma_E}$ is the input alphabet, $\eta\colon Q_E \times \hat{\Sigma_E} \cup \{\varepsilon\} \times 2^{Q_E}$ is the transition function, and $F_E$ is the set of final states.
This NFA has a linear number of states and transitions in the number of input symbols in $E$.
In this construction, for every indexed action $a_i \in \hat{\Sigma}$ there exists exactly one transition labeled by $a_i$, i.e.\ there is a unique state, denoted by $q_i \in Q_E$, such that $q_i \in \eta(p,a_i)$ for some state $p \in Q_E$, and furthermore there is only one such $p$.
Given an index $i$, we can easily determine the state $q_i$ in constant time.

A pair $b = (i,s)$ of an index $i$ and a vertex $s \in \mit{Vertices}$, where $i$ is the index of a modifier $a_i$, is called a \emph{modifier application}.
For a semi-state $S$ and a vertex $s \in \mit{Vertices}$, by $S^s$ we denote the semi-state $S$ with the current position set to $s$.
By $S\cdot(i,s)$, we denote the semi-state after applying the modifier $a$, where $a_i$ occurs in $\hat{E}$, to $S^s$, i.e.\ the semi-state $(S^s)\cdot a$.

For a semi-state $S$ and a regular expression $E$, we define the \emph{play graph} $\mathcal{G}(S,E) = (\mathcal{V},\mathcal{E})$, which is a directed graph constructed as follows:
\begin{itemize}[wide]
\item $\mathcal{V} = \{(p,s,(b_1,\ldots,b_h)) \mid p \in Q_E, s \in \mit{Vertices}, h \in \mathbb{N}, \forall_i\ b_i\text{ is a modifier application}\}$.
Every vertex $v \in \mathcal{V}$ stores a state $p$ of the NFA, a position $s$ on the board, and a sequence of modifier applications.
It corresponds to the semi-state $(S\cdot b_1\cdot\ldots\cdot b_h)^s$, which is denoted shortly by $S(v)$.
If $p=q_i$ for some $i$, then it also corresponds to the game state $(S(v))_i$.
\end{itemize}
To define edges, we consider a vertex $v = (p,s,(b_1,\ldots,b_h))$ and the outgoing transitions from $p$ in the NFA.
Note that if $p=q_i$, then the labels of these transitions are available actions under $E$ for game states with rules index $i$.
Thus if they are valid, then they are also legal for the game state $(S(v))_i$.
Thus the edges $\mathcal{E}$ are defined by:
\begin{itemize}[wide]
\item $(p,s,(b_1,\ldots,b_h)) \to (p',s,(b_1,\ldots,b_h,(i',s))) \in \mathcal{E}$, if $p' \in \eta(p,a_{i'})$ and $a$ is a valid modifier for $S(v)$.

\item $(p,s,(b_1,\ldots,b_h)) \to (p',s',(b_1,\ldots,b_h)) \in \mathcal{E}$, if $p' \in \eta(p,a_{i'})$, $a$ is a valid non-modifier for $S(v)$, and $s'$ is the current position in $S(v)\cdot a_{i'}$ (can be different from $s$ when $a_{i'}$ is a shift action).

\item $(p,s,(b_1,\ldots,b_h)) \to (p',s,(b_1,\ldots,b_h)) \in \mathcal{E}$, if $p' \in \eta(p,\varepsilon)$.
\end{itemize}
From the construction, and because the NFA recognizes $\mathcal{L}(\hat{E})$, it follows that from a game state $(S^s)_i$ we can reach $((S')^{s'})_{i'}$ if and only if in the play graph from the vertex $(q_i,s,())$ we can reach a vertex $(q_{i'},s',(b_1,\ldots,b_h))$, where the modifier applications are such that $(S^s)\cdot b_1\cdot\ldots\cdot b_h = (S')^{s'}$.

For the given game state $I_r = I^s_r$, the algorithm performs a DFS in the play graph $\mathcal{G}(I,\mit{Rules})$ starting from the vertex $(q_r,s,())$.
Moreover, we visit together all the vertices with the same modifier applications.
Whenever we reach a switch, i.e., a vertex whose state is $q_i$ where $i$ indexes a switch, we backtrack and report the move $b_1,\ldots,b_h$ in this vertex.
Note that if the description is $k$-straight, we always have at most $k+1$ modifier applications (together with the final switch).
The algorithm is shown in Alg.~\ref{alg:report-moves}.

\begin{algorithm}[htb]
\caption{The algorithm from Theorem~\ref{thm:computing-moves1}.}\label{alg:report-moves}
\begin{algorithmic}[1]
\Procedure{ReportAllMoves}{semi-state $S$, state $q$ of $\mathcal{N}_\mit{Rules}$, modifier applications list $B$}
\State{Create the directed graph $\mathcal{G}(S,E)$ restricted to vertices $(p,s,B)$.}
\State{Perform a DFS for $\mathcal{G}(S,E)$ from $(q,s)$, where $s$ is the position in $S$.}
\ForAll{visited $(p,c,B)$}
  \ForAll{valid modifier $a_i$ for $S^{c}$ with $\{p'\}=\eta(p,a_i)$}
    \State{$S' \gets S \cdot a$}
    \State{$B.\Call{Append}{(i,a)}$}
    \If{$a_i$ is a switch}
      \State{Report move $B$}
    \Else
      \State{\Call{ReportAllMoves}{S',p',B}}
    \EndIf
    \State{$B.\Call{Pop}{ }$} \Comment{Remove $(i,a)$}
  \EndFor
\EndFor
\EndProcedure
\end{algorithmic}
\end{algorithm}

Except for patterns, determining the validity of actions is trivial and doable in linear time in their text representation lengths (under the assumption of constant time arithmetic for evaluating comparisons and assignments).
We can maintain counters of pieces in constant time for arithmetic expressions.
For patterns, we use a similar algorithm recursively.
To check if a pattern $\{?\,M\}$ or $\{!\,M\}$ is valid at a vertex $v=(q_i,s,(b_1,\ldots,b_h))$, we use the NFA $\mathcal{N}_M$ and the play graph $\mathcal{G}(S(v),M)$ to check if there is a path from the initial vertex $(q'_0,s,())$ to a vertex with a final state in $F_M$.
Note that if the description is $k$-straight, the modifier application sequences of reachable vertices are of length at most $k-h$.

\noindent\emph{Complexity}.
The sum of the numbers of states and the number of transitions in all created NFAs (for $\mit{Rules}$ and all patterns) is $\O(|\mathcal{R}|)$.
Applying a non-pattern action takes linear time in its representation length, but we can do a simplifying trick to think it is computed in constant time. 
If the representation is longer than a constant, we add more states and transitions to the NFA that form a chain.
In every transition, we store only a part of the action of constant length.
The sum of the number of states and transitions is larger but still bounded by $\O(|R|)$, and since such a chain is always traveled as a whole, we can visit every state and travel every edge in amortized constant time.

To count visited vertices, observe that every vertex corresponds uniquely to some state $p$ of one of the NFAs, a position $s$, and a sequence $b_1,\ldots,b_g$ applied so far to the semi-state $I$ (concatenated sequences from each recursive call, i.e.\ $(I\cdot b_1\cdot\ldots\cdot b_g)^s$ is the semi-state of the vertex).
In a $k$-straight game, this sequence has length at most $k+1$, and if the length is equal to $k+1$, then the last modifier applied is a switch, which determines $p$ and $s$.
There are at most $|\mit{Rules}|\cdot|\mit{Board}| \le |\mathcal{R}|\cdot |\mathcal{S}|$ possibilities for $b_i$ and also for $p$ together with $s$.
So the number of all possibilities sums up to $\O((|\mathcal{R}|\cdot|\mathcal{S}|)^{k+1})$ (we can assume $|R| \ge 2$), which bounds the number of all reachable vertices.
The sum of the weights of all the traveled edges is also bounded by $\O((|\mathcal{R}|\cdot|\mathcal{S})^{k+1})$.

However, we may visit multiple times the same vertex of the play graph of a pattern.
For the same sequence of modifier applications, a pattern can be evaluated no more than $\O(|\mathcal{S}|)$ times, once for each different position.
Hence, a pattern nested at a depth $g$ (where depth $0$ corresponds to $\mit{Rules}$) can be evaluated no more than $|\mathcal{S}|^g$ times.
Thus, in the worst case, we may visit every vertex or travel the same edge $\O(|\mathcal{S}|^d)$ times, since $d$ is the maximum nesting depth of a pattern.

Altogether, we get the upper bound $\O((|\mathcal{R}|\cdot|\mathcal{S}|)^{k+1} \cdot |\mathcal{S}|^d)$ on the number of visited vertices and traveled edges.

Of course, technically in the algorithm we do not have to store the graph explicitly and can maintain one list of modifier applications modified accordingly by appending and removing the last element, which allows processing a vertex or a transition in constant time plus the time required to apply an action and check its validity.
Reporting a single move takes $\O(1)$ time if we yield just the pointer to the temporary list or use a persistent list implementation.
Hence, the size of the graph is also an upper bound on the running time.

For the upper bound on the space complexity, we notice that in a single recursive call of \textbf{ReportAllMoves} we use up to $\O(|\mathcal{R}|\cdot |\mathcal{S}|)$ memory to store the respective part of the graph $\mathcal{G}(S,E)$ and the auxiliary array to perform DFS.
If there are any patterns to check on the edges in the current call, we create their graphs and NFAs, but their total size, including the main graph, do not exceed $\O(|\mathcal{R}|\cdot |\mathcal{S}|)$.
As there can be at most $k$ recursive calls and every graph from a finished call is no longer needed, we use at most $\O(|\mathcal{R}|\cdot |\mathcal{S}|\cdot k)$ memory.
\end{proof}

The result from Theorem~\ref{thm:computing-moves1} can be improved.
This is mostly a theoretical improvement due to the involved construction and a large hidden constant in the time complexity.
In practice, usually the maximum depth of patterns is small, and the running time depends on the exact number of reached vertices, which can be considerably smaller than the worst case bound.

\medskip
\begin{theoremnum}[\ref{thm:computing-moves2}]
For a given $k$-straight description $(k \ge 1)$, the set of all legal moves can be found in $\O((|\mathcal{R}|\cdot|\mathcal{S}|)^k)$ time and in $\O(k(|\mathcal{R}|\cdot|\mathcal{S}|))$ space.
\end{theoremnum}
\begin{proof}
To improve the algorithm from Theorem~\ref{thm:computing-moves1}, we need to avoid repeating evaluations of patterns. Then the time complexity will be the same as the upper bound on the visited vertices and transitions.

First, we assume that we know the value of $k$.
To check if a pattern $\{?\,M\}$ or $\{!\,M\}$ is applicable at a vertex $v=(p,s,(b_1,\ldots,b_h))$, we evaluate it at once for all vertices $(q'_0,s',())$, where $q'_0$ is the initial state in the NFA $\mathcal{N}_M$, for every position $s'$, instead of only for $s$.
In this way, we evaluate a pattern for every sequence $(b_1,\ldots,b_h)$ at most once instead of $\O(|\mathcal{S}|)$ times, and we will have stored the evaluation results.
We use the NFA $\mathcal{N}_M$ and the play graph $\mathcal{G}(S(v),M)$.
First, if $h < k$, for every vertex $(q'_i,s',((i',s')))$ of the play graph, where $a_{i'}$ is an action occurring in $\hat{M}$, $q'_i$ is the vertex of the NFA after reading $a_{i'}$, and $s'$ is a position, we check whether starting from $(q'_i,s',((i',s')))$ a vertex with a final state is reachable.
This is done by the usual algorithm with DFS.
Because possibly the semi-states of these starting vertices may be not reachable, the straightness does not bound the number of possibly applied modifiers.
However, if we know the straightness $k$, then we can limit the number of applied modifiers to be at most $k-h$, hence if the number of applied modifiers exceeds this bound, we know that the starting vertex cannot be reachable, as this would contradict the straightness.
The special case is when $k=h$ since we cannot consider vertices with one modifier application; we then skip these vertices.
Starting from these vertices, together with vertices $(f,s',())$ for every final state $f$ of $\mathcal{N}^M$, we propagate the marking backwards over the vertices with an empty sequence of modifier applications by a DFS
in the inverse graph of $\mathcal{G}(S(v),M)$.
Since the semi-states of these vertices are known to be $(S(v))^{s'}$ for some position $s'$, it is easy to check the validity of every action when considering backward transitions.
The pattern evaluates to true if and only if $(q'_0,s,())$ is marked.

The above modification requires to know the value of $k$.
However, the algorithm can also report if the given $k$ is not enough (the straightness is larger).
If the number of modifier applications applied to the first semi-state becomes larger than $k$, not counting the final switch, then $k$ is not enough.
Doing this for patterns is more complicated, because we start from vertices $(p,s',((i',s')))$ which are not necessarily reachable from certain vertices $(q'_0,s',())$, thus the upper bound from straightness may not apply to them.
For a vertex $(p,s',((i',s')))$, if we have encountered that the limit of modifier applications is not enough, we propagate this information backwards to vertices $(q'_0,s',())$.
If the pattern is needed to be evaluated for such a vertex, this means that the limit of modifier applications can be exceeded from it.

Now, we can run the algorithm for $k=1,2,\ldots$ until the play graph can be successfully traversed with the limit of $k$ modifier applications.
In each iteration, the total number of visited nodes and vertices is upper bounded by $\O((|\mathcal{R}|\cdot|\mathcal{S}|)^{k+1})$, and this is also the upper bound on the running time.
For a $k$-straight description, the bound for all iterations is $\sum_{i=1}^{k} \O((|\mathcal{R}|\cdot|\mathcal{S}|)^{i+1})$, which is asymptotically equal to $\O((|\mathcal{R}|\cdot|\mathcal{S}|)^{k+1})$.
The space complexity remains the same under these modifications.
\end{proof}

\begin{theorem}\label{thm:rbg-legal-move}
For unrestricted RBG, Problem~1 is PSPACE-complete, even if $\mathcal{R}$ is fixed.
\end{theorem}
\begin{proof}
It is easy to see that Problem~1 is solvable in NPSPACE and thus in PSPACE after the following observation.
Since the number of different game states is exponential (in $|\mathcal{R}|$ and $|\mathcal{B}|$), we can guess and store the length $\ell$ of a legal move, and then guess consecutive actions online without storing them. The same applies recursively to patterns.

To show PSPACE-hardness, we reduce from a canonical problem whether a given Turing machine accepts using a polynomially bounded tape.
Additionally, we assume that the given machine has a counter so that it never cycles: either it accepts or terminates by not having a legal transition (going outside the tape is also considered as an illegal transition).
We can assume that the given machine is simulated by a fixed universal Turing machine (UTM), which has a polynomial overhead.
Since the UTM is fixed, only the tape is given as the input, which stores the code of the given machine and working space.

The rules will be fixed as they will depend on only the UTM. The tape will be encoded on the board.

We can assume that there are only two tape symbols $0$ and $1$.
The UTM has the unique initial state $q_0$ and the unique accepting state $q_f$.
At the beginning, the UTM is in state $q_0$ and at the left-most position on the tape. The initial tape content is given as the input.
The rules of the UTM are of the form $(q,s,q',s',D)$, where $q,q'$ are states, $s,s' \in \{0,1\}$ are tape symbols, and $D \in \{\mathrm{Left},\mathrm{Right}\}$ is the direction of the head movement.
Such a rule indicates that if the UTM is in state $q$ and reads symbol $s$ at the head position, then the UTM changes the current state to $q'$, writes symbol $s'$ at the head position, and moves the head to the direction specified by $D$.

We reduce to the problem of an existence of a legal move for a given RBG description $\mathcal{R},\mathcal{S}$.
There is only one player $\mit{utm}$ and no variables except the score variable, which is irrelevant for the problem.
For every state $q$ and tape symbol $s$, we create the pieces $p_{q,s}$ and $p_{s}$. They represent the content of a tape cell with symbol $s$, where $p_{q,s}$ additionally denotes that the head is at this position and the machine is in state $q$.

The board represents the tape.
It has $n$ vertices, where $n$ is the tape length. All of the vertices have $\mathit{right}$ and $\mathit{left}$ outgoing edges except the left-most and right-most ones, respectively.
The initial semi-state contains the left-most vertex as the current position, and the pieces are assigned according to the input tape content.

It remains to define $\mit{Rules}$.
For every transition rule $r=(q,s,q',s',\mathrm{Left})$ and for every tape symbol $t$, we create the subexpression $E_{r,t} = \{p_{q,s}\}\ [p_{s'}]\ \mit{left}\ \{p_{t}\}\ [p_{q',t}]$.
This is constructed analogously for the rules with direction $\mathrm{Right}$.
It is easy to see how these subexpressions relate to the transition rules of the UTM. First, we check if the state of the machine is $q$, we are at the position of the head, and the current symbol is $s$, then we replace the current symbol with $s'$, move the head, and place the head without changing the symbol $t$ there. Note that there are only $2m$ such subexpressions, where $m$ is the number of the transition rules of the UTM.
Finally, $\mathit{Rules}$ are defined as follows:
$$
((\mit{left}^*+\mit{right}^*)\ (E_{r_1,0}+E_{r_1,1}+\hdots+E_{r_n,0}+E_{r_n,1}))^*\ 
\{p_{q_f,0},p_{q_f,1}\}
{\to\mit{utm}}
$$
The $(\mit{left}^*+\mit{right}^*)$ part allows to find the cell (vertex) with the current head state (a piece $p_{q,s}$), which is unique in every reachable game state. The outermost star is responsible for repeating the rules of the UTM as long as they are valid for the current configuration.
The part $\{p_{q_f,0},p_{q_f,1}\}$ checks if the header has reached the accepting state $q_f$.

To see that the RBG description is proper, recall that the simulated machine includes a counter, thus a configuration cannot be repeated, and only a bounded number of rules can be applied (the bound depends exponentially on $n$). Hence also the maximum number of legally applied modifiers is bounded, which satisfies condition~(1).
The keeper is not used here, thus condition~(2) is trivially satisfied.

It is obvious that if it is possible to reach the final switch $\to utm$, then the UTM accepts in state $q_f$.
To prove the other implication, we note that before applying the actions from $E_{r,t}$, there is exactly one piece of the form $p_{q,s}$ on the board. Since the machine is deterministic, this implies that there is at most one $E_{r,t}$ defining valid subsequence of actions. If there is no such $E_{r,t}$, then there is no valid transition of the UTM.
\end{proof}

\begin{theorem}\label{thm:rbg-winning}
For unrestricted RBG and strongly $1$-straight RBG, Problem~2 is EXPTIME-complete, even if $\mathcal{R}$ is fixed.
\end{theorem}
\begin{proof}
Since the number of different game states reachable from a given game state is at most exponential, we can compute the whole game tree in EXPTIME and evaluate the winner strategy at each node by the min-max algorithm.

To show EXPTIME-hardness, we use the result from \cite{Fraenkel1981Computing}, where it was shown that for generalized chess to arbitrary sizes of boards and their setups (i.e., the board and pieces assignment is the input), determining whether the first player has a winning strategy is EXPTIME-complete. 
There is presented a reduction from the Boolean game problem, which is known to be EXPTIME-complete \cite{Stockmeyer1979Provably}.
For the reduction, there are used only queens, rooks, bishops, pawns, and two kings.
If both players play optimally, the first (white) player either wins or the second player forces a draw due to repeating configurations.
There is no possibility of castling nor en~passant.
Pawns are deadlocked, in the sense that a player that allows being captured by an opponent's pawn loses.
Stalemate is not possible, provided the players play optimally.

We can easily encode the above rules in RBG; in fact, we can encode the full rules of chess.
Moreover, for the necessary subset of rules mentioned above, we can construct a strongly $1$-straight description by putting after each non-switch modifier an extra switch that does not change the player, and ending the game when capturing the opponent's king.
We finally need to add a turn counter (as a variable), to ensure that the game ends after some number of steps; as the bound, we can set the number of possible board configurations, which is exponential in the board size.
Moreover, the rules can be fixed as do not depend on the given board, which proves EXPTIME-hardness also in this case.
\end{proof}

\begin{theorem}
For unrestricted RBG and strongly $1$-straight RBG, Problem~3 is PSPACE-complete, even if $\mathcal{R}$ is fixed.
\end{theorem}
\begin{proof}
To show that the problem is PSPACE, we need to solve two problems in polynomial space: to check if the straightness is finite (condition~(1)) and to check whether for every reachable state there is exactly one keeper completion (condition~(2)).

We will show that the first problem is in co-NPSPACE, thus in PSPACE.
First, note that if there is an unbounded number of modifiers in legal sequence of actions for $\mit{InitialState}$ under $\mit{Rules}$, then in a long enough such a sequence, the same modifier is applied twice to the same game state (the same semi-state and the same rules index).
Since the number of game states is exponential in (in $|\mathcal{S}|$ and $|\mathcal{R}|$), it must happen for a legal sequence of at most exponential length.
We can guess their positions $n_1$ and $n_2$ in the sequence.
Then we can guess a legal sequence and apply it, storing only the current game state, the action, and the counter.
We check if the $n_1$-th and the $n_2-th$ applied modifiers are indeed the same and are applied to the same game state.

The second problem can also be solved in co-NPSPACE.
If there are two distinct keeper completions for some reachable game state, they can be reached after traversing at most all the game states, which is at most exponential.
First, we can guess a legal sequence of actions to reach such a game state.
Then, we can just guess the first sequence of legal actions (along with the keeper switches) to some game state (keeper completion), and then the second one to another keeper completion.
We check if these game states are different.

To show PSPACE-completeness, we can use a similar construction as in Theorem~\ref{thm:rbg-legal-move}.
Instead of ending $\mit{Rules}$ with $\to utm$, we end it with $[x]^*$, for an arbitrary piece $x$.
Now, the input machine reaches the accepting state $q_f$ if and only if we can start applying legally the actions of $[x]^*$, which implies that the straightness of the application language is infinite, thus the description is not proper.
There is no risk of description improperness caused by multiple keeper completions because there are no keeper moves.
To make the description strongly $1$-straight, we just add a switch right after every switch; this may add more leaves in the game tree, but does not affect the properness of the description.
The rest of the reasoning remains the same.
\end{proof}

\newpage
\section{Technical specification}

First, we define \emph{low level RBG} (\emph{LL-RBG}) which directly corresponds to the abstract RBG definition. Then, we present \emph{high level RBG} (\emph{HL-RBG}) which makes the description more human-readable, self-commenting, and concise. This split separates the human-readable interface from the base language that is mainly aimed to be machine-processable. Thus, the low-level representation is more suited for an RBG player, a game manager, or a compiler; while the high level version is better for human usage, especially writing game rules. 
Also, this allows modifications of high-level constructions and improving the HL-RBG language without changing the underlying machine implementation.

\subsection{Lexical analysis}

Both RBG language levels are interpreted as a sequences of tokens. 
Identifiers and predefined keywords are case-sensitive. C-style (\lstinline{//}) and C++-style (\lstinline{/**/}) comments are permitted.
The following grammar describes all possible tokens:

\begin{mdframed}[style=grammar]
\begin{grammar}
<alpha> ::= any ASCII letter a--z or A--Z

<digit> ::= any digit 0--9

<alphanum> ::= <alpha> | <digit>

<ident> ::= <alpha> \{ <alphanum> \}

<nat> ::= <digit> \{ <digit> \}

<token> ::= <ident> | <nat>
\alt `(' | `)' | `\{' | `\{?' | `\{!' | `\{$' | `\{' | `[' | `[$' | `]'
\alt `\textasciitilde' | `\#' | `-' | `+' | `^' | `/' | `*' | `,' | `;' | `:' | `$'
\alt `=' | `->' | `->\,\!>' | `!' | `?' | `!=' | `==' | `<' | `<=' | `>' | `>='
\alt `players' | `pieces' | `variables' | `rules' | `board' | `hexagon' | `rectangle' | `cuboid'
\end{grammar}
\end{mdframed}

The tokens are parsed greedily, i.e., the longest possible token is taken from the prefix of the string to be parsed. The token parsing ends at the first whitespace, comment, or character that is not a valid continuation of the current token. 
Then, additional whitespaces and comments are skipped, and the next token is being parsed.
Thus,
\begin{lstlisting}
->>
\end{lstlisting}
is parsed as single token, even though parsing it as \lstinline{->} followed by \lstinline{>} would also make sense. When the parts of the string are separated by a white space or comment, as in
\begin{lstlisting}
->/**/>
\end{lstlisting}
they are interpreted as two separate tokens.
After transforming an input string into a token string, any whitespaces and comments are discarded.

\subsection{Low level RBG}

Game description in Regular Boardgames contains exactly one definition of every section (in any order): \lstinline{#board}, \lstinline{#players}, \lstinline{#variables}, \lstinline{#pieces} and \lstinline{#rules}. Each section begins with a corresponding \lstinline{#}-prefixed keyword, and ends with the beginning of the next section or end of the game description file. 

\subsubsection{Declaration sections}

We need to define sets of $\mathit{Pieces}$, $\mathit{Variables}$, and $\mathit{Players}$ that will be used in the game. The part of the grammar describing declaration of these sections is: 

\begin{mdframed}[style=grammar]
\begin{grammar}
<piece-name> ::= <ident>

<pieces-section> ::= `# pieces =' <piece-name> \{ `,' <piece-name> \}

<variable-name> ::= <ident>

<variable-bound> ::= <nat>

<bounded-variable> ::= <variable-name> `(' <variable-bound> `)'

<variables-section> ::= `# variables =' [ <bounded-variable> \{ `,' <bounded-variable> \} ]

<players-section> ::= `# players =' <bounded-variable> \{ `,' <bounded-variable> \}
\end{grammar}
\end{mdframed}

The \lstinline{#pieces} section contains all pieces that can potentially appear on the board. Both \lstinline{#variables} and \lstinline{#players} sections define assignable variables and their $\mathit{Bound}$ in brackets. Sets of identifiers declared in above-mentioned sections have to be disjoint. Note that \lstinline{#variables} section has to be defined even if the $\mathit{Variables}$ set is empty. The identifier \lstinline{keeper} is allowed as a player name. Such a player is completely distinct from the \emph{keeper} -- special player described in the abstract definition. (This identifier is also allowed in every part of the description without creating any inconsistency.)

Initial values of all variables, including players, are set to $0$. The starting player is always the keeper.

\noindent\textbf{Example}:

\begin{lstlisting}
#players = white (100), black (100)
#pieces = whitePawn, blackPawn, empty
#variables = // no non-player variables
\end{lstlisting}

\subsubsection{Board section}

The purpose of the \lstinline{#board} section is not only to define the $\mathit{Board}$ graph, but also to establish the initial $P$, that is pieces assignment in the initial position.

\begin{mdframed}[style=grammar]
\begin{grammar}
<label> ::= <ident>

<node-name> ::= <ident>

<edge> ::= <label> `:' <node-name>

<node> ::= <node-name> `[' <piece-name> `]' `\{' <edge> \{ `,' <edge> \} `\}'

<board-section> ::= `# board =' <node> \{ <node> \}
\end{grammar}
\end{mdframed}

For every vertex, we define its name, a starting piece, and the set of outgoing edges given as label names and target nodes. Labels of edges outgoing from a single vertex are unique, but target nodes may be not. Loop edges can also be defined. Every starting piece has to be explicitly defined in the \lstinline{#pieces} section. The order of nodes or outgoing edges within a single node is not important, except that the first defined vertex is the starting vertex, i.e., $s$ component of the starting $\mathcal{S}$.

The set of edges labels (known as $\mathit{Dirs}$ in the abstract definition) is disjoint with the $\mathit{Pieces}$, $\mathit{Variables}$ and $\mathit{Players}$ sets. Technically the graph does not have to be connected but, due to the RBG semantics, only one connected component can be used by the game.

\noindent\textbf{Example}:

\begin{lstlisting}
#board =
    v11 [whitePawn] {up: v12, right: v21}
    v21 [whitePawn] {up: v22, right: v31, left: v11}
    v31 [whitePawn] {up: v32, left: v21}
    v12 [empty] {up: v13, right: v22, down: v11}
    v22 [empty] {up: v23, right: v32, left: v12, down: v21}
    v32 [empty] {up: v33, left: v22, down: v31}
    v13 [blackPawn] {right: v23, down: v12}
    v23 [blackPawn] {right: v33, left: v13, down: v22}
    v33 [blackPawn] {left: v23, down: v32}
\end{lstlisting}

The example describes a simple $3\times 3$ square board. The upper row is filled with black pawns while the lowermost with the white ones. Square \lstinline{v11} is the starting square.

\subsubsection{Rules section}

The regular expression over actions alphabet, describing the game $\mathit{Rules}$, has to be encoded according to the following grammar:

\begin{mdframed}[style=grammar]
\begin{grammar}
<shift> ::= <label>

<on> ::= `\{' [ <piece-name> \{ `,' <piece-name> \} ] `\}'

<off> ::= `[' <piece-name> `]'

<rvalue> ::= <sum>

<sum> ::= <sum-element> \{ <next-sum-elements> \}

<next-sum-elements> ::= `+' <sum-element> | `-' <sum-element>

<sum-element> ::= <multiplication-element> \{ <next-multiplication-elements> \}

<next-multiplication-elements> ::= `*' <multiplication-element> | `/' <multiplication-element>

<multiplication-element> ::= `(' <sum> `)' | <nat> | <variable-name>

<assignment> ::= `[$' <variable-name> `=' <rvalue> `]'

<comparison-operator> ::= `>' | `>=' | `==' | `!=' | `<=' | `<'

<comparison> ::= `\{$' <rvalue> <comparison-operator> <rvalue> `\}'

<switch> ::= `->' <player-name> | `->\,\!>'

<move-check> ::= `\{?' <rule> `\}' | `\{!' <rule> `\}'

<action> ::= <shift> | <on> | <off> | <assignment> | <comparison> | <switch> | <move-check>

<rule> ::= <rule-sum>

<rule-sum> ::= <rule-sum-element> \{ `+' <rule-sum-elements> \}

<rule-sum-element> ::= <rule-concatenation-element> \{ <rule-concatenation-element> \}

<rule-concatenation-element> ::= <action> <potential-power> | `(' <rule-sum> `)' <potential-power>

<potential-power> ::= `*' | `'

<rules-section> ::= `# rules =' <rule>
\end{grammar}
\end{mdframed}

\vspace*{10px}
The \lstinline{#rules} section directly corresponds to the regular expression in abstract RBG, which describes legal moves. Operators are left associative.
Precedence of the operators from the highest:
\begin{enumerate}
\item power operator (\verb!*!),
\item concatenation (no operator),
\item sum (\verb!+!).
\end{enumerate}
Every expression is indexed as described in the abstract RBG definition. The index for the initial game state is $0$.

\noindent\textbf{Example}:

\begin{lstlisting}
#rules = ->white // hite player moves
    (
        (up* + down*)(left* + right*) // seek any square
        {whitePawn}[empty] // pick up my piece by replacing it with empty
        // go diagonally left, diagonally right or straight up
        (
        	up left {empty, blackPawn} // capturing, if diagonally
          + up {empty}
          + up right {empty, blackPawn}
        )[whitePawn]
        [$ white=1][$ black=0] // set white win, in case black ends with no pieces
        (
            {! up}->>{} // if white piece reached board top, it won
          + {? up}->black // turn control to black otherwise
        )
        // black player moves, analogously
        (up* + down*)(left* + right*)
        {blackPawn}[empty]
        (
        	down left {empty, whitePawn}
          + down {empty}
          + down right {empty, whitePawn}
        )[blackPawn]
        [$ black=1][$ white=0]
        (
            {! down}->>{}
          + {? down}->white
        )
    )*
\end{lstlisting}

The example presents regular expression for the game breakthrough. Note that the rules are independent on the board size. Thus, the board can be as large as desired and rules can remain the same as long as $\mathit{Dirs}$ does not change.

One could point out that the example shows much redundancy. In particular, it does not take advantage of the game symmetry. One of the benefits of the high level RBG is that it mitigates this lack.

\subsection{High level RBG}

High level RBG provides a few additional features making game descriptions easier to read, understand, and modify. Every low level RBG description is also HL-RBG.

\subsubsection{Macro system}

High level RBG introduce a simplified C-like macro system. Macro definitions follow the grammar:

\begin{mdframed}[style=grammar]
\begin{grammar}
<macro-name> = <ident>

<macro-argument> = <ident>

<token-string> ::= potentially empty string of tokens

<macro> ::= `#' <macro-name> [ `(' <macro-argument> \{ `;' <macro-argument> \} `)' ] `=' <token-string>
\end{grammar}
\end{mdframed}

Syntactically macros look like sections known from the low level RBG. However, their names need to be distinct from any predefined sections: \lstinline{board}, \lstinline{players}, \lstinline{variables}, \lstinline{pieces}, and \lstinline{rules}. Macros definitions can be interleaved with sections.

We can divide macros into those with brace-enclosed arguments and those without them. Two macros, one with arguments and one without them, cannot have the same name. However, two macros with arguments can share the same name, as long as they differ in the number of arguments.

\noindent\textbf{Example}:

\begin{lstlisting}
#m1 = x // legal
#m2(a) = a // legal
#m2(a;b) = a+b // legal
#m2(b) = b // illegal
#m2 = x // illegal
\end{lstlisting}

All occurrences (instantiations) of the macro in sections or other macros after its definition are replaced with the token string specified in the macro's definition. 
Argument instances required for macros with arguments can be any strings of tokens, including the empty string. While inserting a value of the macro, all arguments identifiers are replaced by their instances.

\noindent\textbf{Example}:
\begin{lstlisting}
#m0 = m1 // will be interpreted as a plain "m1" string, macro m1 is defined later
#m1 = x
#m2 = m1 // will be replaced with "x"
#m3(a;b) = a + b
#m4 = m3(x;y) // will be replaced with x + y
#m5 = m3(;) // will be replaced with a single +
#m6 = m3 // will be interpreted as plain "m3", no arguments given
#m7 = m1(x) // will be interpreted as "x (x)"
\end{lstlisting}

Many potentially redundant parts of RBG description differs only in the piece color or movement direction. To avoid repeating such code, the macro system comes with the metaoperator \lstinline{~}. During macro instantiation, it concatenates together two or more tokens into one, similarly to the C language \lstinline{##} operator. Both arguments of this operator and its result have to be legal RBG tokens.

\noindent\textbf{Example}:
\begin{lstlisting}
#m1 = x~y
#m2 = m1 // will be replaced with "xy"
#m3(a;b) = a~b
#m4 = m3(x;y) // will be replaced with "xy"
#m5 = m3(8;y) // "8y" is not valid RBG token
#m6 = m3(1;2) // will be interpreted as "12"
#m7 = m3(x~y;z) // will be interpreted as "xyz"
\end{lstlisting}

Using this macro system, we can rewrite breakthrough \verb!#rules! section into better-organized form:

\begin{lstlisting}
#anySquare = (up* + down*)(left* + right*)

#pickUpPiece(color) = {color~Pawn}[empty]

#basicMove(color;oppositeColor;forwardDirection) =
    pickUpPiece(color)
    (
    	forwardDirection left {empty, oppositeColor~Pawn}
      + forwardDirection {empty}
      + forwardDirection right {empty, oppositeColor~Pawn}
    )[color~Pawn]

#endGame = ->>{}

#checkForWin(color;oppositeColor;forwardDirection) =
    [$ color=1][$ oppositeColor=0]
    (
        {! forwardDirection} endGame
      + {? forwardDirection}->oppositeColor
    )

#fullMove(color;oppositeColor;forwardDirection)=
    anySquare
    basicMove(color;oppositeColor;forwardDirection)
    checkForWin(color;oppositeColor;forwardDirection)

#rules = ->white
    (
        fullMove(white;black;up)
        fullMove(black;white;down)
    )*
\end{lstlisting}

Because of the defined macros and their names, comments become practically unnecessary to understand the description. Also, it is much easier to change the rule of movement, so that, e.g.\ pawns can move in all directions.

\subsubsection{Predefined boards}

A great majority of boardgames utilizes only a small set of boards shapes, like rectangles or hexagons. Describing their whole structure as a graph can be tedious and error-prone. The high level RBG defines a few board generators that translate popular board shapes into RBG-compatible graphs. They can be placed in the \lstinline{#board} section instead of the low level RBG graph description. The generators allow a user to define edge names, leaving the node names unspecified.

The first predefined board is \lstinline{rectangle}. It follows the grammar:

\begin{mdframed}[style=grammar]
\begin{grammar}
<board-line> ::= `[' <piece-name> \{ `,' <piece-name> \} `]'

<rectangle> ::= <board-line> \{ <board-line> \}

<rectangle-board> ::= `rectangle (' <label> `,' <label> `,' <label> `,' <label>' `,' <rectangle> `)'
\end{grammar}
\end{mdframed}

All board lines contain the same number of pieces. The edge labels passed to the generator are names of up, down, left, and right directions respectively. Thus, the \lstinline{#rules} section can use them as shifts. The left-most top square is taken as the starting vertex.

Given the \lstinline{rectangle} generator and the macro system, we can rewrite graph representing $3\times3$ rectangle board using much simpler syntax:

\begin{lstlisting}
#line(piece) = [piece,piece,piece]

#board =
    rectangle(up,down,left,right,
        line(blackPawn)
        line(empty)
        line(whitePawn)
    )
\end{lstlisting}

Extending this graph is as simple as inserting more \lstinline{piece}s into the \lstinline{line} macro to increase width, or inserting into the generator more \lstinline{line} instantiations to increase height.

The names of pieces in the last argument of \lstinline{rectangle} can be omitted. In such case, the formula describes a graph that lacks some vertices. This may be useful in defining boards with the standard neighborship structure, but with a non-standard shape.

\textbf{Example}.

\begin{lstlisting}
#board =
    rectangle(up,down,left,right,
        [e,e,e]
        [e, ,e]
        [e,e,e]
    )
\end{lstlisting}

The above description will generate a square board without the central vertex. Omitting vertices is a feature that can be used in all generators.

The next type of generator is \lstinline{hexagon}. As the name implies, it generates the board based on hexagonal cells neighboring with six other nodes (or less when at the board edge). The generator is defined by the similar grammar:

\begin{mdframed}[style=grammar]
\begin{grammar}
<hexagon> = <board-line> \{ <board-line> \}

<hexagon-board> ::= `hexagon (' <label> `,' <label> `,' <label> `,' <label>' `,' <label>' `,' <label>' `,'
<hexagon> `)'
\end{grammar}
\end{mdframed}

Boardlines lengths have to form a hexagon. Thus, up to some point subsequent lines should be longer by one than the previous ones. After that point, every line has to be one cell shorter than the previous. According to these restrictions, the following boardlines sequence is a valid last argument of the \lstinline{hexagon} generator:

\begin{lstlisting}
  [e,e,e,e]
 [e,e,e,e,e]
[e,e,e,e,e,e]
 [e,e,e,e,e]
  [e,e,e,e]
   [e,e,e]
\end{lstlisting}

But this one is not:

\begin{lstlisting}
  [e,e,e,e]
 [e,e,e,e,e]
[e,e,e,e,e,e]
[e,e,e,e,e,e] // should have one 'e' less or more
\end{lstlisting}

The label arguments are north-west, north-east, east, south-east, south-west and west directions identifiers respectively.

\textbf{Example}. The description

\begin{lstlisting}
#board = hexagon(northWest, northEast, east, southEast, southWest, west,
  [e,e]
 [e,e,e]
  [e,e])
\end{lstlisting}

is equivalent to:

\begin{lstlisting}
#board =
     v00[e]{east:v10,southEast:v11,southWest:v01}
     v10[e]{southEast:v21,southWest:v11,west:v00}
     v01[e]{east:v11,northEast:v00,southEast:v02}
     v11[e]{east:v21,northEast:v10,northWest:v00,southEast:v12,southWest:v02,west:v01}
     v21[e]{northWest:v10,southWest:v12,west:v11}
     v02[e]{east:v12,northEast:v11,northWest:v01}
     v12[e]{northEast:v21,northWest:v11,west:v02}
\end{lstlisting}

taking into account initial pieces setting, vertices neighborship, and edge labels.

We also defined \lstinline{cuboid}. It helps to define boards shaped like rectangular cuboids. We describe it with the following grammar:

\begin{mdframed}[style=grammar]
\begin{grammar}
<cuboid> ::= `[' <rectangle> \{ <rectangle> \} `]'

<hexagon-board> ::= `cuboid (' <label> `,' <label> `,' <label> `,' <label>' `,' <label>' `,' <label>' `,'
<cuboid> `)'
\end{grammar}
\end{mdframed}

Its label arguments are up, down, left right, front, and back edge labels respectively. The last argument of the generator is the initial pieces setting in the cuboid.

\textbf{Example}. The description piece:

\begin{lstlisting}
#board = cuboid(up, down, left, right, front, back,
	[[e,e,e]
     [e,e,e]]
    [[w,w,w]
     [w,w,w]]
    [[b,b,b]
     [b,b,b]])
\end{lstlisting}

describes a board, where the \lstinline{e} pieces layer is at the back and \lstinline{b} are on the front of the cuboid.

\subsubsection{Syntactic sugar}

High level RBG also provides a few syntactic constructions which, while insignificant, may slightly improve readability of the game rules.

The one which is the most frequently used is the expression \lstinline{expr^n}. Given a positive natural number $n$, it is equivalent to \lstinline{expr} concatenated $n$ times. This syntax can be useful e.g.\ when defining a shift over many vertices in the same direction:

\begin{lstlisting}
#up8Times = up^8
// has the same meaning as:
#up8Times = up up up up up up up up
\end{lstlisting}

Another commonly appearing pattern is the choice of a piece to be placed in off. For example, pawn promotion in chess can be expressed as:

\begin{lstlisting}
#promotePawn(color; forward) = {color~Pawn}{! forward}([color~Knight]+...+[color~Queen])
\end{lstlisting}

High Level RBG introduce another syntactic sugar to ease expressing such constructions. In our example, we may write:

\begin{lstlisting}
#promotePawn(color; forward) = {color~Pawn}{! forward}[color~Knight, ..., color~Queen]
\end{lstlisting}

This abbreviation may be especially useful in conjunction with macros:

\begin{lstlisting}
#majorPieces(color) = color~Knight,...,color~Queen
#promotePawn(color; forward) = {color~Pawn}{! forward}[majorPieces(color)]
\end{lstlisting}

We can also use commas to simplify concatenations of assignments. It is quite usual to assign scores to many players at once:

\begin{lstlisting}
#endGame = [$ player1 = 0][$ player2 = 100][$ player3 = 50]->>{}
\end{lstlisting}

which can be shortened in HL-RBG to:

\begin{lstlisting}
#endGame = [$ player1 = 0,player2 = 100,player3 = 50]->>{}
\end{lstlisting}

Note that commas in offs sequences have other meaning than in assignments sequences; thus both constructions cannot be mixed and used inside single \lstinline{[ ]} brackets.

\end{document}